\documentclass[10pt,journal,compsoc]{IEEEtran}
\ifCLASSOPTIONcompsoc
  \usepackage[nocompress]{cite}
\else
  \usepackage{cite}
\fi

\ifCLASSINFOpdf
\else
\fi

\usepackage{amsmath,amsfonts}
\usepackage{algorithmic}
\usepackage{array}
\usepackage[caption=false,font=normalsize,labelfont=sf,textfont=sf]{subfig}
\usepackage{textcomp}
\usepackage{stfloats}
\usepackage{url}
\usepackage{verbatim}
\usepackage{graphicx}
\def\BibTeX{{\rm B\kern-.05em{\sc i\kern-.025em b}\kern-.08em
    T\kern-.1667em\lower.7ex\hbox{E}\kern-.125emX}}
\usepackage{balance}
\ifCLASSINFOpdf
\else
\fi
\usepackage{amsfonts}
\usepackage{graphicx}
\usepackage{enumerate}
\usepackage{tabu}
\usepackage{algorithm}
\usepackage{algorithmic}
\usepackage{multirow} 
\usepackage{amsmath}
\usepackage{xcolor}
\usepackage{amsmath}
\usepackage{latexsym}
\usepackage{amsmath}
\usepackage{amsthm}
\usepackage{graphicx}
\usepackage{amssymb}
\usepackage{algorithm}
\usepackage{algorithmic}
\usepackage{graphicx}
\usepackage{epstopdf}
\usepackage{float}
\usepackage{booktabs}
\usepackage{multirow}
\usepackage{indentfirst}
\usepackage{subfig}
\usepackage{longtable}
\usepackage{threeparttable}
\usepackage[pdftex,pdfnewwindow=true,linkcolor=blue,breaklinks=true, colorlinks=true,filecolor=blue,anchorcolor=blue]{hyperref}
\usepackage{url}
\usepackage{cases}
\newtheorem{lemma}{Lemma}
\newtheorem{theorem}{Theorem}
\newcommand{\Real}{\mathbb{R}}

\newcommand{\C}{\mathcal{C}}

\newcommand{\U}{\mathbf{U}}
\newcommand{\V}{\mathbf{V}}

\newcommand{\x}{\mathbf{x}}

\newcommand{\vv}{\mathbf{v}}

\newtheorem{definition}{Definition}
\begin{document}
%
\title{Accelerated Fuzzy C-Means Clustering Based on New Affinity Filtering and Membership Scaling}
%
%
%
%

\author{Dong~Li,
        Shuisheng~Zhou,
        and~Witold~Pedrycz,~\IEEEmembership{Life~Fellow,~IEEE}.
\IEEEcompsocitemizethanks{\IEEEcompsocthanksitem D. Li, S. Zhou are with School of Mathematics and Statistics, Xidian University, Xi'an 710071, China  (e-mail: lidong$\_$xidian@foxmail.com; sszhou@mail.xidian.edu.cn).\protect

\IEEEcompsocthanksitem W. Pedrycz is with the Department of Electrical and Computer Engineering,
University of Alberta, Edmonton T6R 2V4 AB, Canada, and with the Systems Research Institute, Polish Academy of Sciences, 00-901 Warsaw, Poland, and also with the Department of Electrical and Computer Engineering, Faculty of Engineering, King Abdulaziz University, Jeddah 21589, Saudi Arabia (e-mail: wpedrycz@ualberta.ca).}
\thanks{Manuscript received xxxx, 2022; revised xxxx, 2022. This work was supported by the National Natural Science Foundation of China under Grants No. 61772020. \emph{(Corresponding author: Shuisheng Zhou.)}}}

%
%

\markboth{IEEE TRANSACTIONS ON KNOWLEDGE AND DATA ENGINEERING,~Vol.~xx, No.~xx, June~2022}%
{Shell \MakeLowercase{\textit{et al.}}: Fuzzy C-Means Clustering Based on New Affinity Filtering and Membership Scaling}
%



\IEEEtitleabstractindextext{
\begin{abstract}
Fuzzy C-Means (FCM) is a widely used clustering method. However, FCM and its many accelerated variants have low efficiency in the mid-to-late stage of the clustering process. In this stage, all samples are involved in the update of their non-affinity centers, and the fuzzy membership grades of the most of samples, whose assignment is unchanged, are still updated by calculating the samples-centers distances. All those lead to the algorithms converging slowly. In this paper, a new affinity filtering technique is developed to recognize a complete set of the non-affinity centers for each sample with low computations. Then, a new membership scaling technique is suggested to set the membership grades between each sample and its non-affinity centers to 0 and maintain the fuzzy membership grades for others. By integrating those two techniques, FCM based on new affinity filtering and membership scaling (AMFCM) is proposed to accelerate the whole convergence process of FCM. Many experimental results performed on synthetic and real-world data sets have shown the feasibility and efficiency of the proposed algorithm. Compared with the state-of-the-art algorithms, AMFCM is significantly faster and more effective. For example, AMFCM reduces the number of the iteration of FCM by 80$\%$ on average.
\end{abstract}

}

\maketitle

\IEEEdisplaynontitleabstractindextext

%
\IEEEpeerreviewmaketitle

\IEEEraisesectionheading{\section{Introduction}\label{sec:introduction}}
\IEEEPARstart{C}{lustering} analysis is one of the important topics in machine learning \cite{jordan2015machine}, which has been widely applied in many fields, including data mining \cite{doring2006data}, pattern recognition \cite{bezdek2013pattern}, image processing \cite{Rezaee2000multiresolution}, etc. The clustering algorithm, which is an unsupervised learning approach, aims to divide the data sets into multiple clusters by similarity measure, among which the data points in the same cluster are similar.

In general, the clustering methods are divided into the hard and soft clustering schemes \cite{Amit2017A,jain2010data}. The representative clustering algorithms are C-Means \cite{Lloyd1982Least} and Fuzzy C-Means (FCM) \cite{bezdek1984fcm}. The hard clustering scheme, in which a sample only belongs to a single cluster, assigns the membership grades between the samples and the clusters as 0 or 1. The hard clustering scheme is very simple and efficient. Inevitably, the hard clustering scheme lacks other distance information except for the closest distance information in the update of the cluster centers, which makes the algorithm more likely to fall into bad local minimum. The soft clustering scheme, in which a sample does not exclusively belong to a single cluster, allows the membership grades to vary between 0 and 1. The soft clustering scheme has better clustering quality because of its flexibility and robustness \cite{zbian}.

The C-Means algorithm (Lloyd algorithm) \cite{Lloyd1982Least} is the most representative method in the hard clustering scheme. However, the computational complexity of all samples-centers distances is very high in C-Means. Thus, many improved  methods have been proposed. Both of these algorithms \cite{elkan2003using, ding2015yinyang} were proposed to speed up C-Means by applying a triangle inequality, which are effectively to avoid unnecessary distance calculations, and achieve higher efficiency. Another trick to deal with this challenge is region division of clusters in clustering. Some related research has been done \cite{Lingras2004Interval, ZHANG2019three, Multi2022, Effective2022}. In recent, ball C-Means \cite{ball2020kmeans} has been proposed to focus on the efficiency of C-Means by reducing the samples-centers distance computations. Significantly, the concept of the neighbor clusters and the partition of cluster are designed to attain the same performance in less time by the multiple novel schemes. 

As one of the most typical soft clustering methods, Fuzzy C-Means (FCM) \cite{bezdek1984fcm} is to divide $n$ samples into $c$ clusters by membership grade matrix $\U$, in which $u_{ij}$ represents the grade $j$th sample belongs to $i$th cluster. FCM is successful in finding and describing overlapped clusters that are ubiquitous in the complex real-world data (see \cite{doring2006data, Rezaee2000multiresolution, Coletta2012} and the references therein). However, all samples-centers distance computations also leads to high computing cost. Meanwhile, all samples are involved in the update of all centers  by the memberships, which leads to  low efficiency of FCM  in the clustering process (see \cite{xu2019robust, Zhou2020A})


In theory, the convergence rate theorem for FCM \cite{hathaway1988recent} is proved, which is that FCM converges linearly to the local minima. Meanwhile, based on the analysis of C-Means\cite{kieffer1982exponential, du1999centroidal, kanungo2000analysis}, it can be found that when C-Means is close to the local minima, the convergence rate of C-Mean drops from an exponential rate to a linear rate. Since both C-Means and FCM are alternating optimization algorithm (AO), Therefore, likewise, the convergence rate of FCM also drops, when FCM is close to a local minima. Meanwhile, the convergence rate of FCM is slower than that of C-Means in the clustering process.


Many researchers have managed to tackle this issue  based on the new update of the centers. Mitra \emph{et al.} \cite{Mitra2006Rough} have designed a Rough-Fuzzy C-Means (RFCM) clustering algorithm, which absorbs the advantages of fuzzy set and rough set and enhances the robustness and efficiency of the fuzzy clustering. Roy and Maji \cite{Roy2020Medical} proposed a spatially constrained Rough-Fuzzy C-Means (sRFCM), which wisely applies the advantages of rough-fuzzy clustering and local neighborhood information together. Furthermore, each cluster was divided into the prossibilistic core region and probabilistic boundary region by sRFCM, which improves the performance of the algorithm. Shadowed sets in the characterization of rough-fuzzy clustering (SRFCM) \cite{zhou2011Shadowed} was introduced to improve the clustering quality and efficiency by optimizing the threshold parameters based on the concept of shadowed set that affect the lower bound and boundary region of each cluster automatically. Similar research can be seen  {\cite{Shadowed2022}, \cite{Particle2022}, \cite{Criterion2022}, \cite{Hybrid2022}, \cite{M3W2022}, \cite{Zhou2018Rough} and the references therein. Unfortunately, unreasonable partition thresholds will result in undesired clustering results. Therefore, the partition parameters need to be optimized per iteration. Inevitably, the computational cost of the selection of the parameters is very high for the region partition.



To solve it, many research improve the performance of FCM by constraining the update of the memberships so that the centers can be updated to their target position more efficiently \cite{Zhao2021, Nie2022, Scalable2022, Novel2022Gu, Semisupervised2022Wang}. Recently, membership scaling Fuzzy C-Means clustering algorithm (MSFCM) \cite{Zhou2020A} has been presented to accelerate the convergence of FCM and maintain high clustering quality, where the in-cluster and out-of-cluster samples are identified by a triangle inequality. Then, the membership grades are scaled to boost the effect of the in-cluster samples and weaken the effect of the out-of-cluster samples in the clustering process.

Although the above-mentioned FCM variants usually improve the efficiency and effectiveness of the algorithms, they ignore the low efficiency in the mid-to-late stage of the clustering process.  There are three reasons: 1) the convergence rate of the alternating optimization algorithm (AO) drops, when the algorithms are in the mid-to-late stage \cite{du1999centroidal}; 2) the FCM variants still need to do a full inverse-distance weighting \cite{xu2019robust}; 3) all samples are still involved in the update of all centers \cite{Zhou2020A}. (see detail analysis in Subsection \ref{subsec3-2}).


%

In this study, we first delve into the relationship between the samples and the centers, and further investigate the characteristic of the clustering process by dividing it into the early stage and the mid-to-late stage. Stemming from those findings, we propose a new accelerated FCM clustering algorithm called \textbf{AMFCM} (\textbf{a}ffinity filtering and \textbf{m}embership scaling based \textbf{FCM}). In the proposed algorithm, a new affinity filtering technique is put forward to precisely identify the complete set of the non-affinity centers of each sample (see \textbf{Definition} \ref{definition1} in Section \ref{sec3}), and a new membership scaling method is suggested to accelerate the whole convergence process of the algorithm.

The main contributions of this paper are as follows:
\begin{enumerate}
\item We design a new affinity filtering scheme, which is composed of $c$ triangle inequalities, to discover all samples-centers affinities. Compared with the previous methods by the triangle inequality in \cite{ding2015yinyang} and \cite{Zhou2020A}, the designed scheme can identify the complete set of the non-affinity center set of each sample more precisely with very low computational complexity. Compared with the works in \cite{Lingras2004Interval, Mitra2006Rough, zhou2011Shadowed, Zhou2018Rough}, the new affinity filtering scheme is parameter-free.
\item We propose a new membership scaling scheme to accelerate FCM convergence, especially in the mid-to-late stage. The new membership scaling scheme sets the membership grades to 0 in the update of the non-affinity centers, which eliminates the effect of the sample on the update of its non-affinity centers and reduces the burden of the fuzzy clusterings in efficiency, and maintains the original update of the remaining centers for each sample.

\item By integrating those schemes with FCM clustering, we propose a new accelerated clustering algorithm called \textbf{AMFCM}, which is the first work that focuses on accelerating the whole convergence process of the FCM-type clusterings.

\item Several experimental results on  synthetic and real-world data sets illustrate that the proposed AMFCM outperforms the state-of-the-art algorithms in efficiency. For example, AMFCM reduces the number of the iteration of FCM by 80$\%$ on average.

\end{enumerate}

The paper is organized as follows.  Section \ref{sec2} presents some preliminaries including notations, C-Means, FCM, and the related clustering algorithm. The research motivation is described in Section \ref{sec3} and a new algorithm is presented in Section \ref{sec4}. The experimental results with discussion are reported in Section \ref{sec5} and Section \ref{sec6} concludes the paper.


\section{preliminaries}\label{sec2}
In this section, some related clusterings are briefly relisted for the convenience of the
following discussion.
\subsection{Notations}
Let a data set be $\mathbf{X}=\{\x_1,\x_2, \cdots, \x_n\}$ with $\x_{j}\in\Real^p$, and the cluster centers be $\V=[\vv_1, \vv_2, \cdots, \vv_c]$, where $\vv_{i}\in\Real^{p}$ is the centroid of the cluster $\C_{i}$ for $i=1,2,...,c$. $t$ is the number of iterations. The distances between $\x_{j}$ and the cluster centers $\V$ are $d_{ij}=\|\x_{j}-\vv_{i}\| (i=1,\cdots,c)$ and they are rearranged in ascending order as $D_{j}^{(1)}\leq D_{j}^{(2)}\leq\cdots\leq D_{j}^{(c)}$.  Displacement of the center $\vv_{i}$ after one update is denoted by $\delta_{i}^{(t)}=d(\vv_{i}^{(t+1)},\vv_{i}^{(t)})$. The membership grade matrix is denoted by $\U=[u_{ij}]\in\Real^{c\times n}$, where $u_{ij}$ represents the grade of $j$th sample belonging to $i$th cluster.

\subsection{C-Means}\label{subsec2-1}
C-Means clustering \cite{Lloyd1982Least}, as the most representative algorithm in the hard clustering, aims to find the $c$ partitions of $\mathbf{X}$ by minimizing the within-cluster sum of the distance from each sample to its nearest center. The underlying objective function is expressed as follows:
\begin{equation}\label{eq_1}
\begin{aligned}
\min_{\U,\V} J_{\textbf{Hard}}(\U,\V)=&\sum_{i=1}^c\sum_{j=1}^n u_{ij}\|\x_{j}-\vv_{i}\|^2,\\
    s.t. \quad &\sum_{i=1}^c u_{ij}=1, u_{ij}= 0 \  \textrm{or} \ 1,
    \end{aligned}
\end{equation}

To solve problem \eqref{eq_1}, which is NP-hard,  C-Means \cite{Lloyd1982Least} consists of two steps: the assignment step assigns each sample to its closest cluster and the update step renews each of the $c$ cluster centers with the centroid of the samples assigned to that cluster. The algorithm repeats those two steps until convergence.
\subsection{Fuzzy C-Means}\label{subsec2-2}
FCM clustering \cite{bezdek1984fcm}, which is a soft clustering, allows a sample to have membership grades in all clusters
instead of exclusively belonging to one single cluster. FCM partitions $\mathbf{X}$ into $c$ clusters by the cluster
centers.  The objective function is expressed as follows:
\begin{equation}\label{eq_fcm}
\begin{aligned}
    \min_{\U,\V} J_{\textbf{Fuzzy}}(\U,\V)=&\sum_{i=1}^c\sum_{j=1}^n u_{ij}^m\|\x_{j}-\vv_{i}\|^2,\\
    s.t. \quad &\sum_{i=1}^c u_{ij}=1, u_{ij}\geq 0,
    \end{aligned}
\end{equation}
with the fuzziness weighting exponent $m>1$.

To optimize problem \eqref{eq_fcm}, FCM usually initializes $\U^{(0)}$, which is a randomly initialized partition matrix,  and updates $\V$ and $\U$ iteratively by
\begin{align}
\vv^{(t+1)}_{i}&=\frac{\sum\limits_{j=1}^{n}\left(u^{(t)}_{ij}\right)^{m}\x_{j}}
{\sum\limits_{j=1}^{n}\left(u^{(t)}_{ij}\right)^{m}},\label{eq_2}\\
u^{(t+1)}_{ij}&=\left[\sum_{k=1}^c\left(\frac{\|\x_{j}-\vv^{(t+1)}_{i}\|}{\|\x_{j}
-\vv^{(t+1)}_{k}\|}\right)^{\frac{2}{m-1}}\right]^{-1},\label{eq_3}
\end{align}
until convergence.

\subsection{Membership Scaling Fuzzy C-Means}\label{subsec2-2}
Membership scaling Fuzzy C-Means clustering algorithm (MSFCM) \cite{Zhou2020A} accelerates the clustering convergence and maintains high clustering quality by using a triangle inequality and membership scaling. Specifically, the triangle inequality \cite{elkan2003using,ding2015yinyang}, which is a tool for mining the samples-centers affinities, is as follows:
\begin{lemma}\label{lem1}
A sample $\x_{j}$ cannot change its nearest cluster after one update, if
\begin{equation}\label{eq_4}
 D_{j}^{(2)}-\max\limits_{1\le i\le c}\delta_{i}\ge D_{j}^{(1)}+\delta_{I_{j}^{*}},
 \end{equation}
where $I_{j}^{*}=\arg{\min\limits_{1\leq i\leq c}\{d_{ij}\}}$.
\end{lemma}
The samples whose closeness relationships do not change after one update are filtered out by using the triangle inequality \eqref{eq_4} and $Q$ is taken as the index set of the filtered samples. The membership grades of the filtered samples are
scaled to accelerate the convergence of FCM.  Therefore, the new update scheme for $\U^{(t+1)}$ is as follows:
\begin{align}
u^{(t+1)}_{i,j}&=\left\{
\begin{array}{ll}
  M_j^{(t)}, & j\in Q^{(t)}, i={I^{*}_j} ^{(t)}, \\
  \beta_j^{(t)} u^{(t)}_{i,j}, & j\in Q^{(t)}, i\neq{{I^{*}_j}^{(t)}}, \\
  u^{(t)}_{i,j}, & j\notin Q^{(t)}, 1\le i\le c, \\
\end{array}
\right.\label{eq_new u}
\end{align}
where $M_j^{(t)}=\left[1+(c-1)\left({D_{j}^{(1)}}/{D_{j}^{(c)}}\right)^{\frac{2}{m-1}}\right]^{-1}$, $\beta_j^{(t)}=\tfrac{1-M_{j}^{(t)}}{1-u^{(t)}_{I_{j}^{*},j}}.$
Then the update of $\V$ in MSFCM is also Eq. \eqref{eq_2}. In \cite{Zhou2020A}, MSFCM also
reduces the participation of the filtered samples in the update of their non-affinity centers and increases the participation of the filtered samples in the update of the remaining centers by scaling the memberships. Therefore, MSFCM has good properties, such as fewer number of iterations, lower time consumption, and higher clustering quality.

\section{Motivation}\label{sec3}
In this section, the relationship between the samples and the centers is  first  described by the following new \textbf{Definition} \ref{definition1} and \ref{definition2}.
\begin{definition}
A cluster center $\vv_{i}$  is the \textbf{non-affinity center} of a sample $\x_{j}$, if $\vv_{i}$ cannot  be the nearest center of $\x_{j}$ in next iteration. Let $\mathcal{P}_{j}$ be the set of the non-affinity centers of $\x_{j}, j=1,2,...,n$.
\label{definition1}
\end{definition}
\begin{definition}
A sample $\x_{j}$ is the \textbf{non-affinity sample} of a cluster center $\vv_{i}$, if $\x_{j}$ cannot
belong to $\vv_{i}$ in next iteration. Let $\C_{i-}$ be the set of the non-affinity samples of $\vv_{i}$, and $\overline{\C_{i-}}$ be the set of the remaining samples, $i=1,2,...,c$.
\label{definition2}
\end{definition}

Note that if a sample $\x_{j}$ is the non-affinity sample of $\vv_{i}$ , $\vv_{i}$ is the non-affinity center of $\x_{j}$, ($i \in \mathcal{P}_{j}$, if $ j \in \C_{i-}$), and vice versa. The current nearest center of $\x_{j}$ can not be the non-affinity center of $\x_{j}$ after one iteration. That is, $|\mathcal{P}_{j}|\leq c-1$. Similarly, $\x_{j}$ can not be the non-affinity sample of the current  nearest center of $\x_{j}$.

For improving the convergence speed and clustering quality, the hierarchy of information granules \cite{Bargiela2008, YaoGranular} guides the algorithms to reduce the contributions of the samples in the update of their non-affinity centers and increase the contributions of the samples in the update of the remaining centers. However, two problems need to be solved. One is how to get the set of the non-affinity centers of each sample efficiently. The other is how to formulate the modification benchmarks for the contributions of the samples in in the update of centers. These two problems motivate the proposal of two new schemes.
\subsection{Searching  Non-Affinity Centers by A New Affinity Filtering}\label{subsec3-1}
In MSFCM, the affinity filtering scheme \eqref{eq_4} can produce the non-affinity centers of the samples with low computations. For any sample $\x_{j}$, \eqref{eq_4} can identify the complete non-affinity centers of $\x_{j}$ if $|\mathcal{P}_{j}|=0$ or $c-1$. However, a set of the non-affinity centers of the sample identified by \eqref{eq_4} is incomplete when $|\mathcal{P}_{j}|\neq 0, c-1$. In order to illustrate this situation, a geometric interpretation is shown in Fig. \ref{fig4}.
\begin{figure}[htp]
 \centering
 \subfloat[]{\label{fig4a}\includegraphics[width=0.24\textwidth]{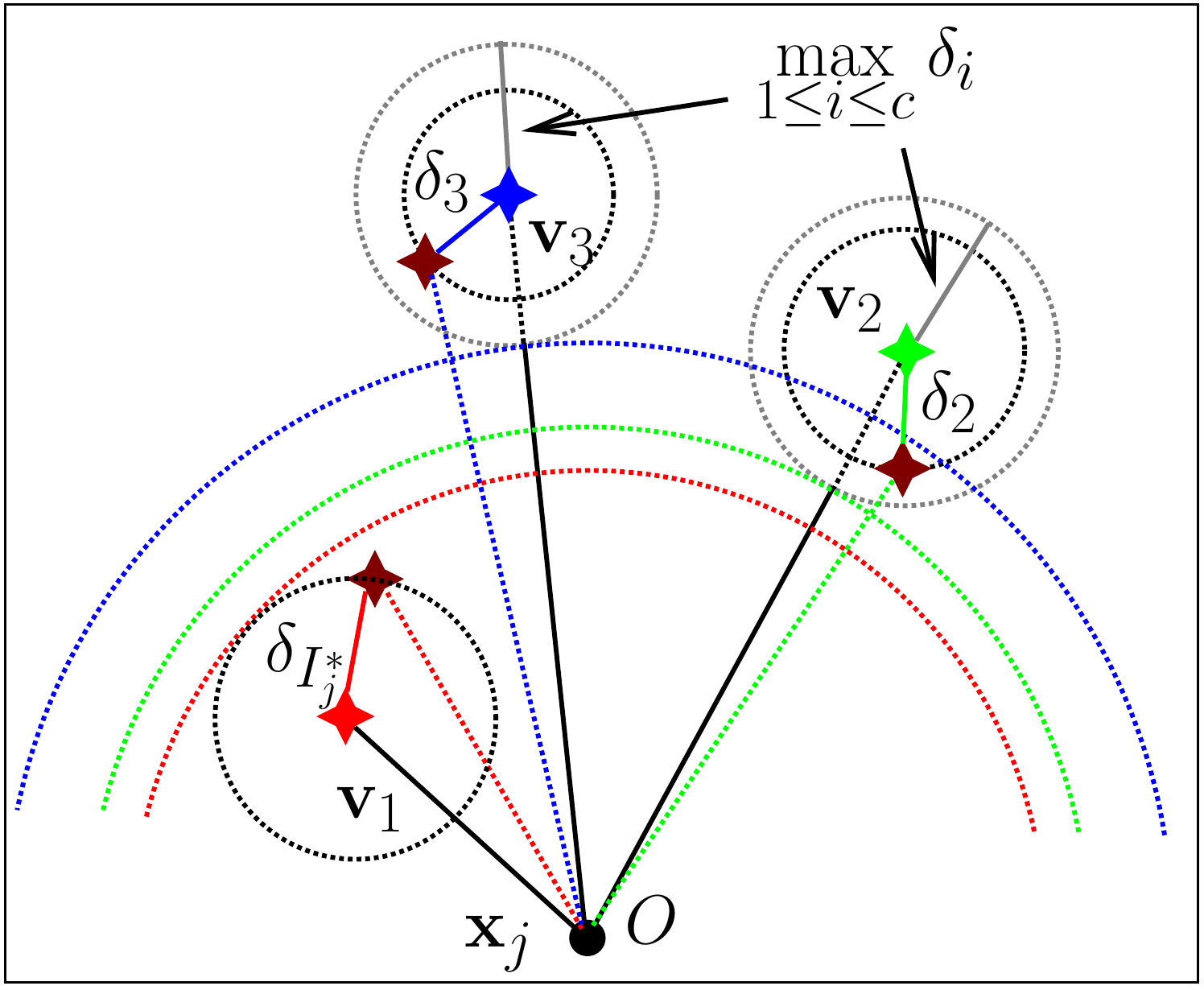}}~~
 \subfloat[]{\label{fig4b}\includegraphics[width=0.24\textwidth]{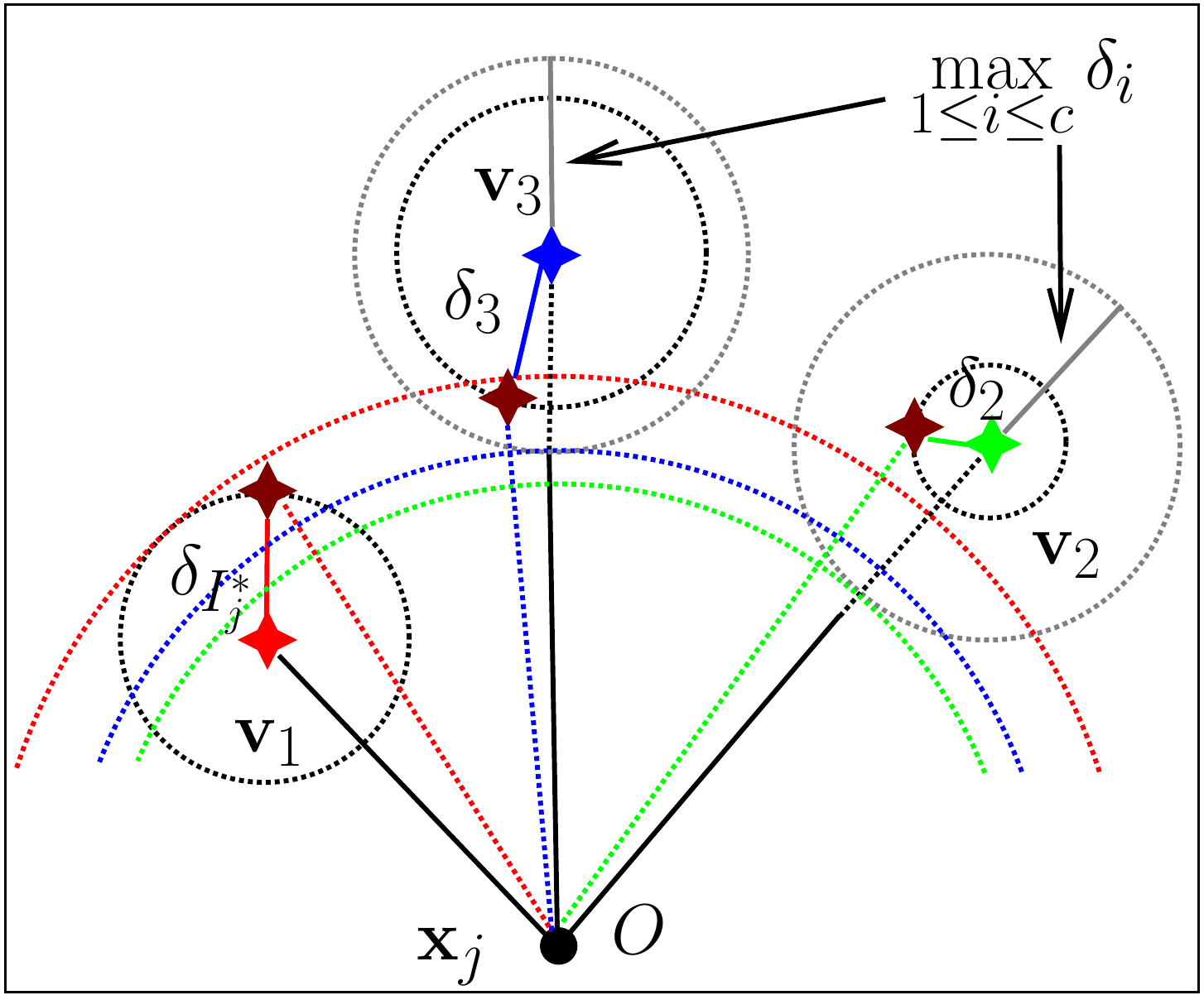}}\\
 \subfloat[]{\label{fig4c}\includegraphics[width=0.24\textwidth]{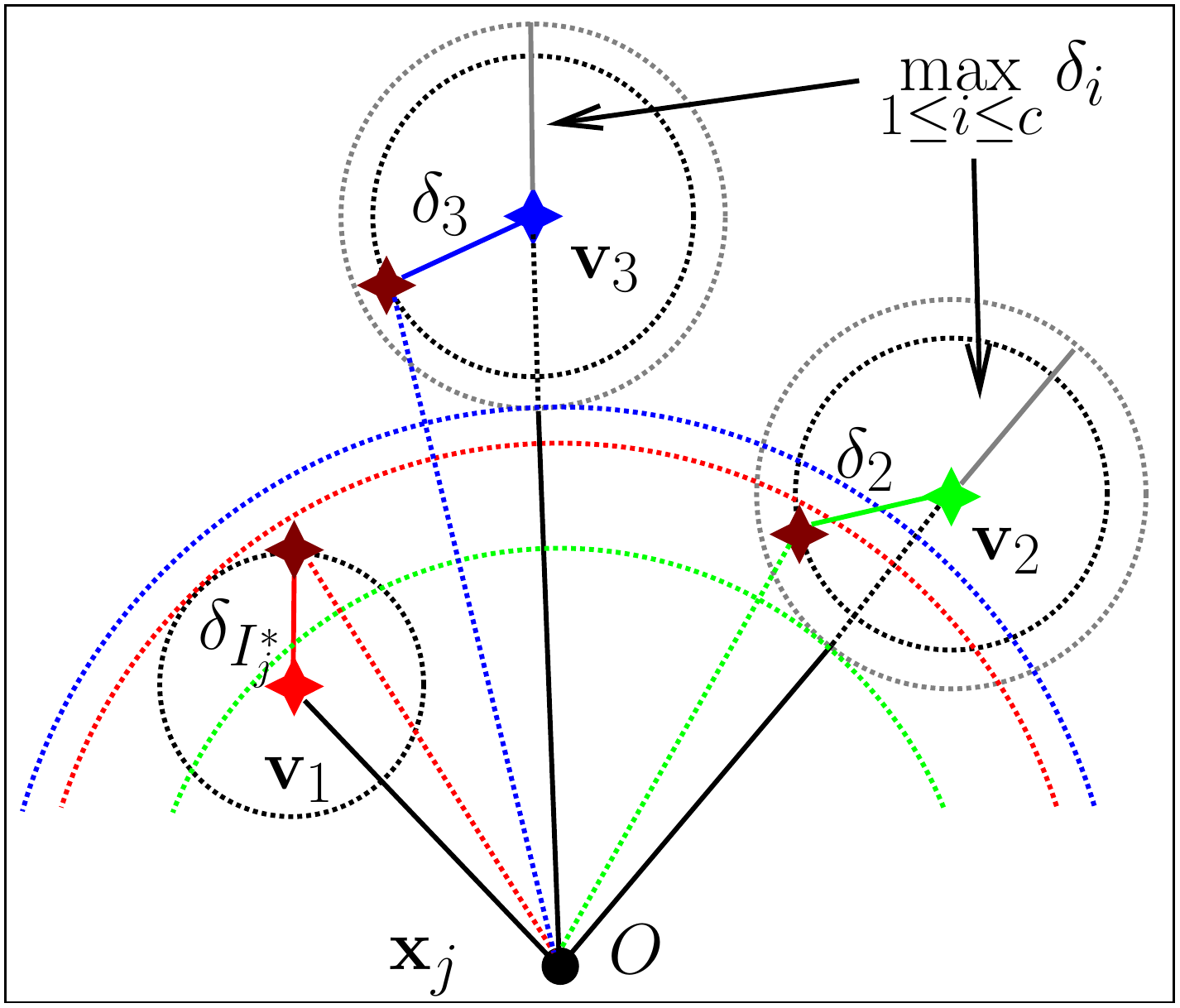}}~~
 \subfloat[]{\label{fig4d}\includegraphics[width=0.24\textwidth]{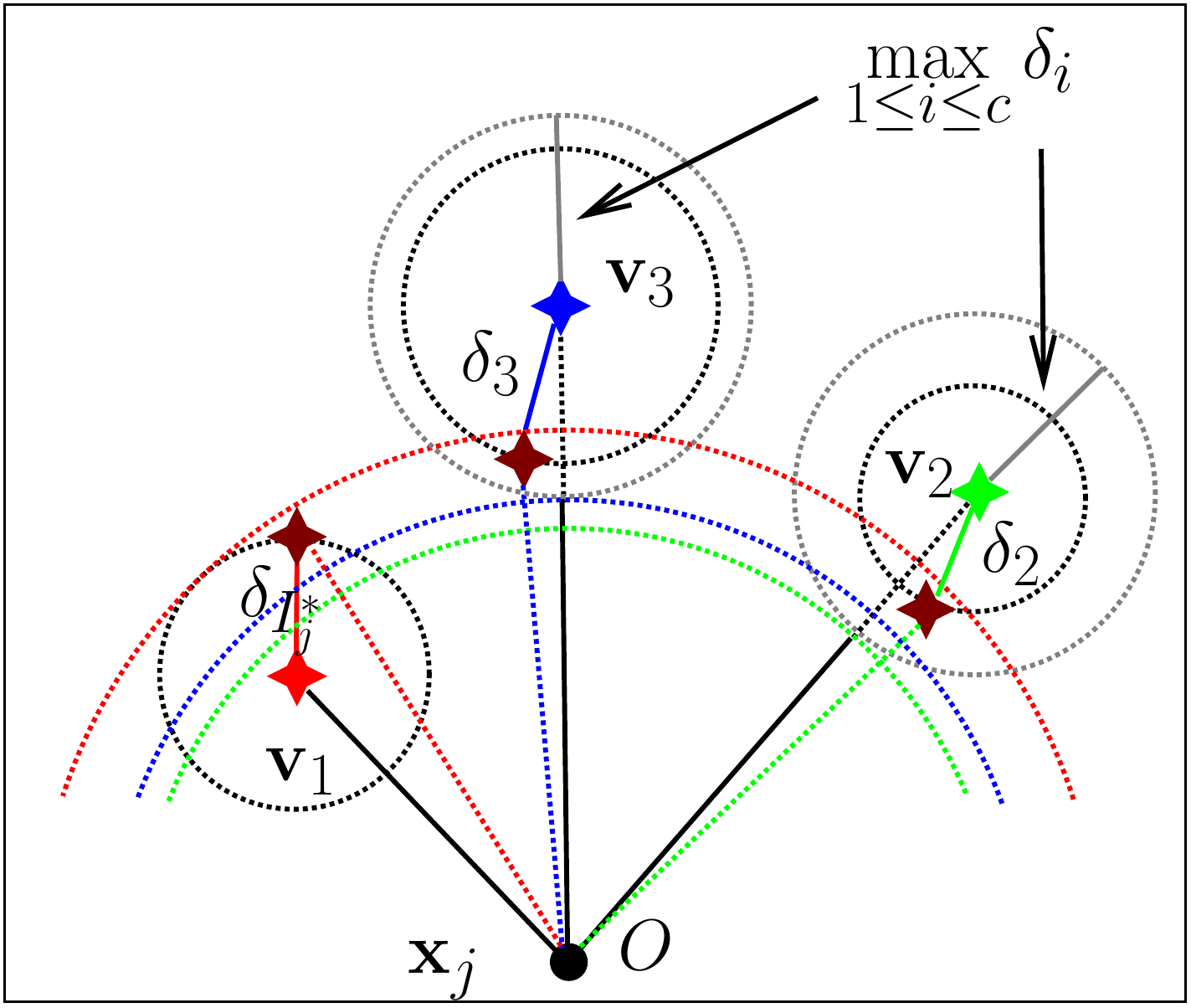}}\\
 \caption{Geometric explanation of identifying the non-affinity centers of $\x_{j}$ by the affinity filtering scheme \eqref{eq_4}. For $\x_{j}$, centers $\vv_{1}$, $\vv_{2}$, $\vv_{3}$ are its nearest, second-nearest, and third-nearest centers, respectively. The brown star points are the possible position of $\vv_{i}$ in next iteration. The radius of the black and gray dot-circles are $\delta_{i}$ and $\max_{1 \leq i \leq c} \delta_{i}$, respectively. In case (\ref{fig4a}), $\vv_{2}$ and $\vv_{3}$ are identified as  the  complete non-affinity centers of $\x_{j}$ by \eqref{eq_4}, where $|\mathcal{P}_{j}|=2$. In case (\ref{fig4b}) and (\ref{fig4c}), the set of the non-affinity centers of $\x_{j}$ is considered empty by \eqref{eq_4}. Actually, the complete set of the non-affinity centers of $\x_{j}$  cannot be accurately identified by \eqref{eq_4} when $|\mathcal{P}_{j}|\neq 0, 2$. For example, $\vv_{2}$ is the non-affinity center of $\x_{j}$ in case (\ref{fig4b}), and  $\vv_{3}$ is the non-affinity center of $\x_{j}$ in case (\ref{fig4c}). In case (\ref{fig4d}), the set of the non-affinity centers of $\x_{j}$ identified by \eqref{eq_4} is complete  because $|\mathcal{P}_{j}|=0$.}\label{fig4}
 \end{figure}

In Fig. \ref{fig4}, centers $\vv_{1}$, $\vv_{2}$, $\vv_{3}$ are its nearest, second-nearest, and third-nearest centers for $\x_{j}$, respectively. Here, $c=3$.  The radius of red dot-arcs is the upper bound of $\vv_{1}$, $d_{1,j}+\delta_{1}$, the radius of green dot-arcs is the lower bound of $\vv_{2}$, $d_{2,j}-\max_{1\le i\le 3}\delta_{i}$, and the radius of blue dot-arcs is the lower bound of $\vv_{3}$, $d_{3,j}-\max_{1\le i\le 3}\delta_{i}$, for $\x_{j}$. Fig. \ref{fig4a} shows the situation, where $|\mathcal{P}_{j}|=2$ for $\x_{j}$. At this time, \eqref{eq_4} ensures that $\vv_{2}$ and $\vv_{3}$ are the complete non-affinity centers of $\x_{j}$. In Fig. \ref{fig4b} and \ref{fig4c}, the set of the non-affinity centers of $\x_{j}$ is considered empty by \eqref{eq_4}. Actually, the complete set of the non-affinity centers of $\x_{j}$ cannot be accurately identified by \eqref{eq_4} when $|\mathcal{P}_{j}|\neq 0, 2$. For example, $\vv_{2}$ is the non-affinity center of $\x_{j}$ in case (\ref{fig4b}), and  $\vv_{3}$ is the non-affinity center of $\x_{j}$ in case (\ref{fig4c}). In Fig. \ref{fig4d}, the set of the non-affinity centers of $\x_{j}$ identified by \eqref{eq_4} is complete because $|\mathcal{P}_{j}|=0$.

For $c \geq 3$, the affinity filtering scheme \eqref{eq_4} always recognizes the incomplete set of non-affinity centers of the sample $\x_{j}$ when $|\mathcal{P}_{j}|\neq 0, c-1$, because  \eqref{eq_4} only employs the lower bound of the second closest center of $\x_{j}$, $D_{j}^{(2)}-\max_{1\le i\le c}\delta_{i}$, to screen all samples-centers affinities. More specifically, the lower bound of its second closest center is used to determine the affinity between $\x_{j}$ and its second closest center, which is also inaccurate. For example, in the case of the same lower bound of $\vv_{2}$, the affinity between $\x_{j}$ and $\vv_{2}$ is different according to \textbf{Definition} \ref{definition1} in Fig. \ref{fig4b} and Fig. \ref{fig4c}. Therefore, the precise identification of all samples-centers affinities should involve the $c$ new lower bounds formed by replacing $\max_{1\le i\le c}\delta_{i}$ with $\delta_{i}$ for $i=1,2,...,c$.

In this paper, a new affinity filtering scheme is proposed, which is composed of $c$ new triangle inequalities. The lower bound of $i$th triangle inequality is $d_{ij}-\delta_{i}$ for $i=1,2,...,c$, as shown in \textbf{Lemma} \ref{lem2}. In conclusion, the new affinity filtering scheme is manipulated to search the complete set of the non-affinity centers of each sample $\x_{j}$ more precisely in any situation, where $0 \leq|\mathcal{P}_{j}|\leq c-1$. To the best of our knowledge, there is no other effort that contributes to discussing the capture of the complete set of the non-affinity centers of each sample in this novel way.
\subsection{Refining the Convergence of FCM}\label{subsec3-2}
In this part, six real-world data sets are clustered by FCM with random initializations. The details of the data sets are shown in Section \ref{sec5}. The iterative fuzzy objective value is applied to analyze the convergence of FCM. The curves of the convergence of FCM with the iteration $t$ are shown in Fig. \ref{fig1}, where the y-axis is the ratio of objectives to the initial value.
\begin{figure}[htp]
\centering
\label{fig3a}\includegraphics[width=0.49\textwidth]{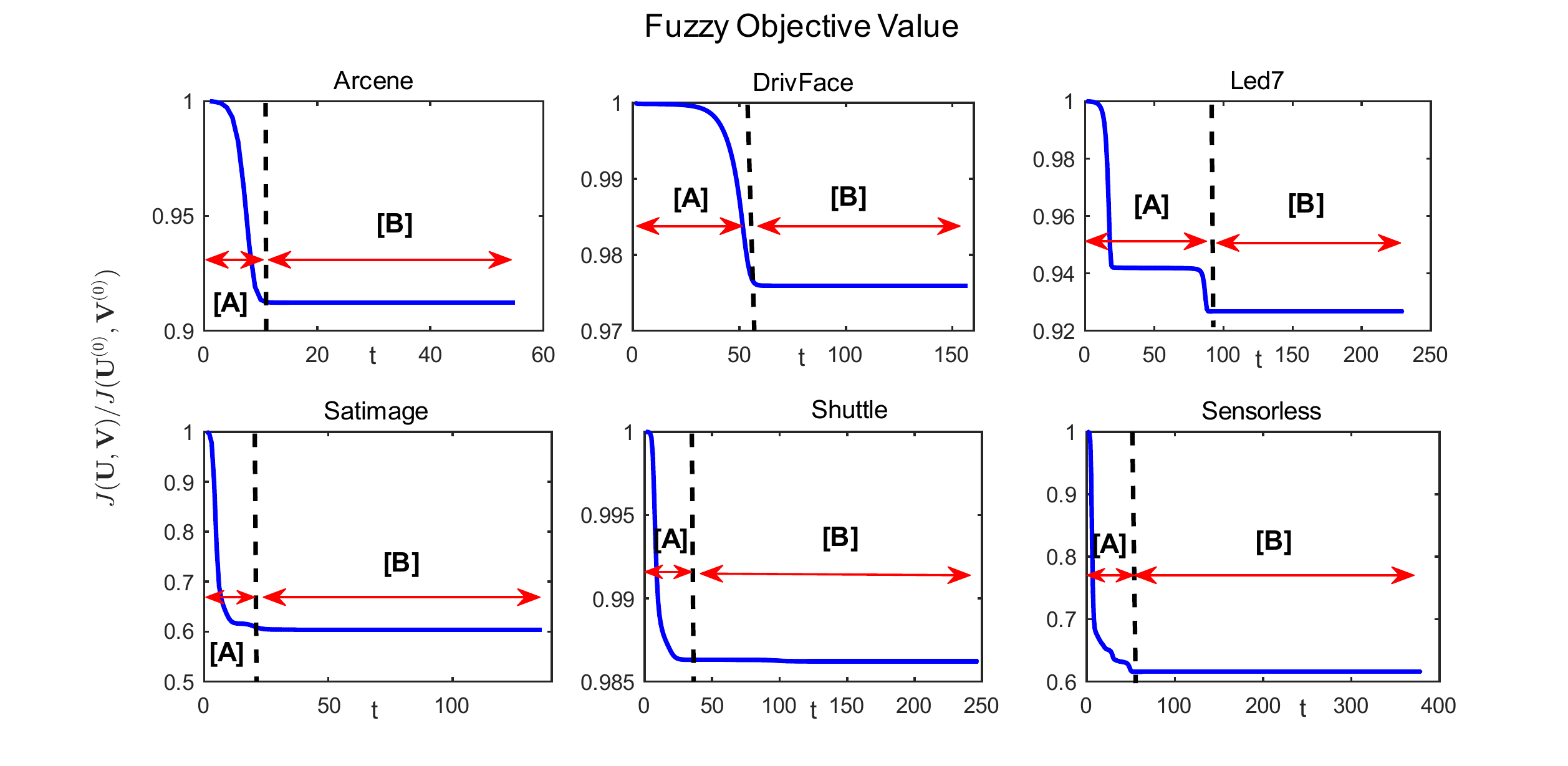}~~~~~~
\caption{Plots of $\frac{J_{\textbf{Fuzzy}}(\U^{(t)}, \V^{(t)})}{J_{\textbf{Fuzzy}}(\U^{(0)}, \V^{(0)})}$ for the iteration $t$ on six real-world data sets with FCM. The initializations is selected randomly for each data set. The plots clearly show that the clustering process of FCM can be divided into stages [A] and [B], where [A] represents the early stage and [B] represents the  mid-to-late stage.}\label{fig1}
\end{figure}

In Fig. \ref{fig1}, it can be observed that the curves of the fuzzy objectives can be divided into stages [A] and [B], where [A] represents the early stage and [B] represents the  mid-to-late stage. At the beginning of the iteration, with random initializations, the belongings of the samples are temporary and uncertain. Therefore, the membership grades are modified for good clustering results, which enlarges the displacement length of the cluster centers to reduce the objective value, as shown in stage [A]. However, the samples in other clusters always continuously interfere with the update of the centers given Eq. \eqref{eq_2}. Once FCM enters stage [B], the convergence efficiency of FCM drops rapidly \cite{du1999centroidal}, which can also be observed that the curves of the fuzzy objectives of FCM are quite stable and long in Fig. \ref{fig1}. According to this observation, it can be inferred that the cluster centers near those target positions and move in a small step in stage [B], where the assignment of most samples will not change, except for the boundary samples between clusters.

To take this question a step further, the update of $\vv_{i}$ is rewritten as $\vv_{i}=(\sum_{j}u_{ij}^{m})^{-1}(\sum_{j\in \C_{i}}u_{ij}^{m}\x_{j}+\sum_{j\notin\C_{i}}u_{ij}^{m}\x_{j})$ in FCM. Clearly,  $(\sum_{j}u_{ij}^{m})^{-1}\sum_{j\in \C_{i}}u_{ij}^{m}\x_{j}$ promotes $\vv_{i}$ to approach its final position, and $(\sum_{j}u_{ij}^{m})^{-1} \sum_{j\notin\C_{i}}u_{ij}^{m}\x_{j}$ prevents $\vv_{i}$ from approaching its final position, which causes the cluster centers still keep fluctuating slightly around the final targets in stage [B]. In this case, although FCM has not converged, the assignment of most samples has been determined. Actually, the fuzzy membership grades are redundant to assess uncertainty for the samples with the unchanging assignment in the current situation. Those membership grades are not only expensive to calculate, but more seriously, the update of those membership grades does not improve clustering quality. To illustrate this problem, a geometric explanation of stage [B] of the convergence process of FCM is presented, as shown in Fig.\ref{fig2}.
\begin{figure}[h]
\centering
\includegraphics[width=0.35\textwidth]{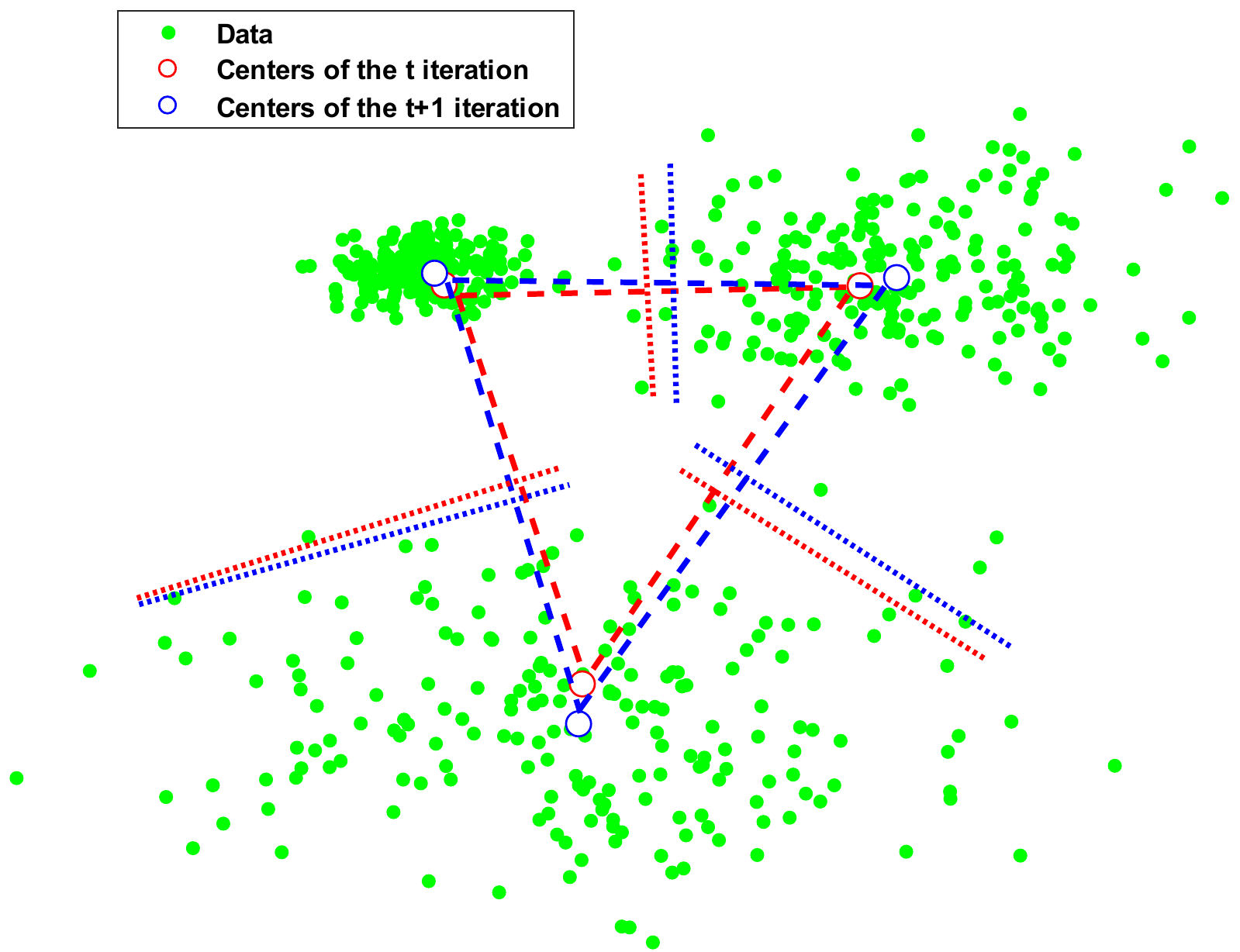}
\caption{Geometric explanation of stage [B] of the convergence process of FCM. The thin red and blue dotted lines are the bisectors between the clusters in these two iterations, respectively. After one iteration in [B] stage, centers still have slight changes around the final targets. Note that the fluctuation of the centers, which hardly affects the convergence result, will occur many times in stage [B]. In this case, the assignment of most samples remains unchanged, and the assignment of some boundary samples, near the thin red and blue dotted lines, is vulnerable to the slight changes of the centers.}\label{fig2}
\end{figure}

In Fig. \ref{fig2}, the assignment of most samples remains unchanged in stage [B]. However, centers always have slight changes around the final targets after one iteration in stage [B] because of the contributions of the samples in other clusters in the update of the centers. Note that the fluctuation of the centers, which hardly affects the convergence result, will occur many times in stage [B], resulting in the slow convergence of stage [B]. In this case, the assignment of some boundary samples, near the thin red and blue dotted lines, is vulnerable to the slight changes of the centers. Therefore, this is the main reason for the slow convergence of stage [B].

Regardless of the defects of the alternating optimization algorithm (AO), the reasons for the slow convergence of FCM can be summarized as the following two points: 
\begin{itemize}
  \item The contributions of the samples in other clusters in the update of the cluster centers continuously delay the whole stages of the convergence process of FCM;
  \item The fuzzy membership grades are redundant to assess uncertainty for the samples with the unchanging assignment in stage [B].
\end{itemize}

As analyzed above, the early stopping in stage [B] can improve the efficiency of the algorithm. An ideal way to do this is to pick an appropriate FCM convergence threshold, $\|\V^{(t+1)}-\V^{(t)}\|$. In fact, the selection of convergence threshold for the early stopping in stage [B] is an extremely difficult task in unsupervised learning. To illustrate this problem, $\|\V^{(t+1)}-\V^{(t)}\|$ for the iteration $t$ on six real-world data sets with FCM are shown in Fig. \ref{fig10}, where the green point is the beginning of stage [B] of FCM. It is found that for any data, the appropriate convergence threshold for the early stopping in stage [B] is difficult to determine without prior information. Therefore, this paper adopts a new feasible approach to effectively accelerate the convergence process of FCM.
\begin{figure}[htp]
\centering
\label{fig10}\includegraphics[width=0.49\textwidth]{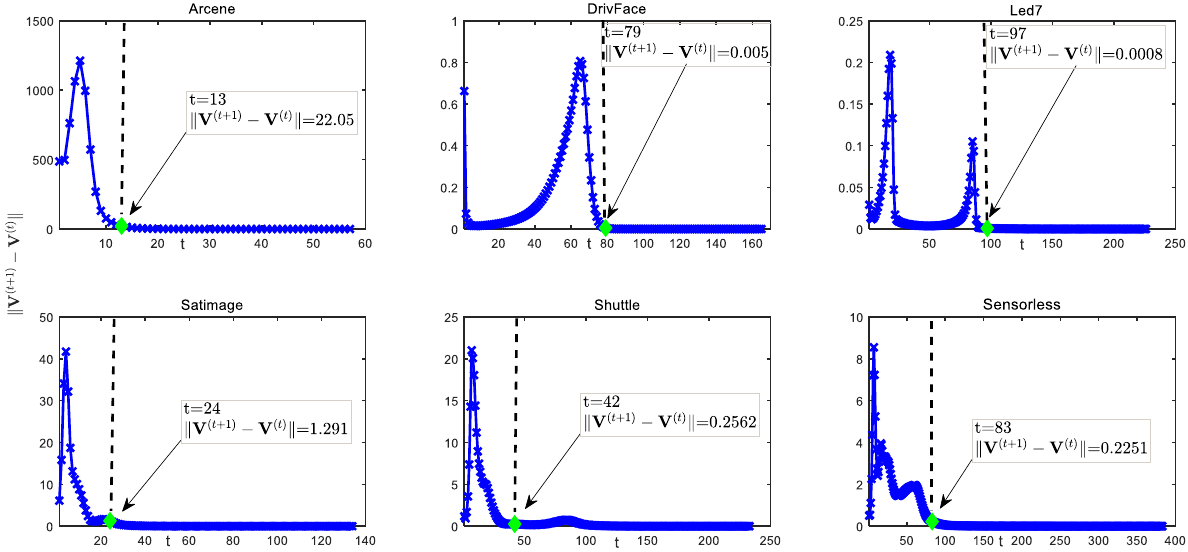}
\caption{Plots of $\|\V^{(t+1)}-\V^{(t)}\|$ for the iteration $t$ on six real-world data sets with FCM.  The plots show that the appropriate convergence threshold for the early stopping in stage [B] is difficult to determine without prior information.}\label{fig10}
\end{figure}

This paper hopes to design a FCM-type clustering algorithm that is a good trade-off between clustering efficiency and quality. According to \textbf{Definition} \ref{definition1} and \ref{definition2},  the update of $\vv_{i}$ can be equivalently written in a novel way as $\vv_{i}=(\sum_{j}u_{ij}^{m})^{-1}(\sum_{j\in \C_{i-}}u_{ij}^{m}\x_{j}+\sum_{j\in\overline{\C_{i-}}}u_{ij}^{m}\x_{j})$. Obviously, in the whole convergence stages of FCM, the contribution of the samples in $\C_{i-}$, $(\sum_{j}u_{ij}^{m})^{-1}(\sum_{j\in \C_{i-}}u_{ij}^{m}\x_{j})$, should be eliminated in the update of $\vv_{i}$ for the  efficiency and quality of the algorithm. This behavior is very significant in improving the efficiency of the algorithm, especially in stage [B], where $\sum_{j\in \C_{i-}}u_{ij}^{m}\x_{j}$ has a large proportion because the assignment of most samples in other clusters does not change. While the contributions of the samples in $\overline{\C_{i-}}$, $(\sum_{j}u_{ij}^{m})^{-1}(\sum_{j\in \overline{\C_{i-}}}u_{ij}^{m}\x_{j})$, should be preserved in the update of $\vv_{i}$, which promotes algorithm convergence and ensures clustering quality. Reasonably, fuzzy membership grades are still meaningful for the samples in $\overline{\C_{i-}}$ in the update of $\vv_{i}$. The details are presented in Section \ref{sec4}.

\section{Accelerated FCM Based on New Affinity Filtering and Membership Scaling}\label{sec4}
In this section, FCM based on new affinity filtering and membership scaling (AMFCM) is proposed to accelerate the whole convergence process of FCM and improve clustering performance. In the proposed AMFCM, a new affinity filtering method is designed to obtain the complete set of the non-affinity information more precisely by a new set of triangle inequalities. Then, a new membership scaling scheme improves the efficiency of FCM convergence in stages [A] and [B]. The operation of the new membership scaling update is two-fold, including the elimination of the contributions of the samples in the update of their non-affinity centers and enhancing the contributions of the samples in the update of the remaining centers.
\subsection{New Affinity Filtering Scheme}\label{subsec4_1}
In the modified C-Means algorithms \cite{elkan2003using,ding2015yinyang}, the triangle inequality that appears in the affinity filtering scheme \eqref{eq_4} is applied to identify all samples-centers affinities. However, the triangle inequality is insufficient for FCM according to the analysis in Section \ref{subsec3-1}. Therefore, a new lemma, which is better applicable to FCM, is presented as follows:
\begin{lemma}\label{lem2}
A cluster center $\vv_{i}$ is the non-affinity center of a sample $\x_{j}$ after one update, if
\begin{equation}\label{eq6}
 d_{ij}^{(t)}-\delta_{i}^{(t)}\geq {D_{j}^{(1)}}^{(t)}+\delta_{I_{j}^{*}}^{(t)}, \quad i\in \{1,2,...,c\},
\end{equation}
where $I_{j}^{*}=\arg{\min\limits_{1\leq i\leq c}\{d_{ij}^{(t)}\}}$.
\end{lemma}
\begin{proof}
In virtue of the triangle inequality for $i\in \{1,2,...,c\}$, there is that
\begin{align*}
d_{ij}^{(t)}-\delta_{i}^{(t)} &= \|\x_{j}-\vv_{i}^{(t)}\|-\|\vv_{i}^{(t+1)}-\vv_{i}^{(t)}\|
&\leq d_{ij}^{(t+1)}.
\end{align*}

Similarly, there is that
\begin{align*}
{D_{j}^{(1)}}^{(t)}+\delta_{I_{j}^{*}}^{(t)}=\|\x_{j}-\vv_{I_{j}^{*}}^{(t)}\|+\|\vv_{I_{j}^{*}}^{(t)}-\vv_{I_{j}^{*}}^{(t+1)}\|
\geq  d_{I_{j}^{*}j}^{(t+1)}.
\end{align*}

If $d_{ij}^{(t)}-\delta_{i}^{(t)}\geq {D_{j}^{(1)}}^{(t)}+\delta_{I_{j}^{*}}^{(t)}$ holds, we have $d_{ij}^{(t+1)} \geq  d_{I_{j}^{*}j}^{(t+1)}$. Therefore, $\vv_{i}$ cannot be the nearest center of $\x_{j}$ after one update. According to \textbf{Definition} \ref{definition1}, we can conclude that $\vv_{i}$ is the non-affinity center of $\x_{j}$.
\end{proof}

Compared with the previous \textbf{Lemma} \ref{lem1}, which can filter the complete set of the non-affinity centers only when $|\mathcal{P}_{j}|=0$ or $c-1$ for each sample $\x_{j}$, \textbf{Lemma} \ref{lem2} provides a more efficient and precise affinity filtering scheme for exploring any situation, where $0 \leq|\mathcal{P}_{j}|\leq c-1$.

According to \textbf{Lemma} \ref{lem2}, a new affinity filtering scheme, which consists of $c$ new triangular inequalities that are the same as \eqref{eq6}, is proposed to more precisely search the complete set of the non-affinity centers of each sample $\x_{j}$ by employing the lower bounds of all centers in any situation, where $0 \leq|\mathcal{P}_{j}|\leq c-1$. Furthermore, the new affinity filtering scheme is parameter-free, which does not need to seek extra thresholds, which determine the non-affinity centers of the samples. A geometric interpretation for the new affinity filtering scheme \eqref{eq6} is given in \textbf{Appendix} \ref{sec7-1}.

In this part, \textbf{Lemma} \ref{lem2} provides a new affinity filtering scheme, which can search the complete non-affinity centers of each sample in any situation, where $0 \leq|\mathcal{P}_{j}|\leq c-1$.  As analyzed above, the computational cost of the new affinity filtering condition is very low, because the additional calculations concern only the displacements of the centers. Next, a new membership scaling scheme is proposed to accelerate the whole convergence stages of FCM in Subsection \ref{subsec4-2}.

\subsection{ New Membership Scaling Scheme}\label{subsec4-2}
As mentioned above in Subsection \ref{subsec3-2},  the update of $\vv_{i}$ is written as $(\sum_{j}u_{ij}^{m})^{-1}(\sum_{j\in \C_{i-}}u_{ij}^{m}\x_{j}+\sum_{j\in\overline{\C_{i-}}}u_{ij}^{m}\x_{j})$. Based on the new affinity filtering scheme \eqref{eq6}, which can autonomously identify the complete set of the non-affinity center of each sample $\x_{j}$, $\mathcal{P}_{j}$, per iteration without much computational cost and extra thresholds, it can be concluded that $\x_{j} \in \C_{i-}$, if $i \in \mathcal{P}_{j}$; otherwise, $\x_{j} \in \overline{\C_{i-}}$, $i=1,2,...,c$.

In this paper, the contributions of the samples in the update of their non-affinity centers should be eliminated and the contributions of the samples in the update of the remaining centers should be increased by the membership grades for the clustering efficiency and quality of the algorithms. Therefore, a new membership scaling technique is cleverly designed to better integrate all samples-centers affinities. In the new membership scaling scheme, the membership grades are set to 0 to eliminate the contributions of the samples in the update of their non-affinity centers, which reduces the burden of the fuzzy clusterings in efficiency. And the fuzzy membership grades still be applied for the remaining centers. Here, the new membership scaling scheme for $\tilde{\U}^{(t)}$ is
\begin{align}
\tilde{u}^{(t)}_{ij}&=\left\{
\begin{array}{ll}
\left[\sum_{k\notin\mathcal{P}_{j}^{(t)}}\left(\frac{d_{ij}^{(t)}}{d_{kj}^{(t)}}\right)^{\frac{2}{m-1}}\right]^{-1}, & i\notin \mathcal{P}_{j}^{(t)}, \\[1cm]
\qquad \qquad \qquad 0 , & i \in \mathcal{P}_{j}^{(t)}. \\
\end{array}
\right.\label{eq_modified u}
\end{align}

Note that the update of $\tilde{u}_{ij}$ is equivalent to the normalization of the membership grades, $\tilde{u}^{(t)}_{ij}=u^{(t)}_{ij}/\sum_{k\notin\mathcal{P}_{j}^{(t)}} u^{(t)}_{kj}$, for $i\notin \mathcal{P}_{j}$. The update of $\V$ is the same way as FCM (Eq. \eqref{eq_2}). Obviously, this scheme is simple and efficient. The required distances between the samples and the current centers, $d_{ij}^{(t)}$, have been calculated. Therefore,  the additional calculations concern only the displacements of the centers in the determination of $\mathcal{P}_{j}$ based on \textbf{Lemma} \ref{lem2}. The complexity analysis will be shown in Subsection \ref{subsec4-3}.
\subsection{ The Proposed Algorithm}\label{subsec4-3}
Accelerated FCM based on new affinity filtering and membership scaling (\textbf{AMFCM}) integrating with traditional one is herein proposed. In this new algorithm, after a traditional FCM iteration, the current $\U$ is adjusted by using the new affinity filtering and membership scaling scheme. The proposed algorithm is presented as follows in Algorithm \ref{alg1}.
\begin{algorithm}[htp!]
	\caption{\textbf{AMFCM}}
	\label{alg1}
	\begin{algorithmic}[1]
        \REQUIRE Date set $\mathbf{X}=\{\x_1,\x_2, \cdots, \x_n\}$, cluster number $c$, fuzzy exponent $m$, and convergence threshold $\varepsilon$;
		\ENSURE Cluster center $\V$.
			\STATE Initialize cluster centers $\V^{(0)}$ and set $t:=0$;
        \STATE Compute $d_{ij}^{(t)}=\|\x_{j}-\vv^{(t)}_{i}\|$, $i=1,...,c, j=1,...,n$;\label{step3}
        \STATE Compute $\U^{(t)}$ with $u^{(t)}_{ij}=\left[\sum_{k=1}^c\left(d^{(t)}_{ij}/d^{(t)}_{kj}\right)^{\frac{2}{m-1}}\right]^{-1}$;\label{step5}
        \STATE Compute  $\bar{\V}^{(t+1)}$ with $\bar{\vv}_{i}^{(t+1)}=\frac{\sum_{j=1}^{n}\left(u^{(t)}_{ij}\right)^{m}\x_{j}}{\sum_{j=1}^{n}\left(u^{(t)}_{ij}\right)^{m}}$;\label{alg_lin0}
            \STATE Compute $\delta_{i}^{(t)}=\|\bar{\vv}_{i}^{(t+1)}-\vv_{i}^{(t)}\|$ for $i=1,2,...,c$; \label{alg_lin1}

		\FOR{$j=1$ to $n$}\label{alg_lin2}
                \STATE  $\mathcal{P}_{j}^{(t)}=\{1 \leq i \leq\ c \mid d^{(t)}_{ij}-\delta^{(t)}_{i} \geq {D_{j}^{(1)}}^{(t)}{}+\delta^{(t)}_{I_{j}^{*}}\}$;
                 \STATE Compute $\tilde{u}_{ij}^{(t)}$ according to Eq. \eqref{eq_modified u};\label{alg_lin4}\\
		\ENDFOR\label{alg_lin3}
			\STATE  Compute $\V^{(t+1)}$ with $\vv_i^{(t+1)}=\frac{\sum_{j=1}^{n}\left(\tilde{u}^{(t)}_{ij}\right)^{m}\x_{j}}
{\sum_{j=1}^{n}\left(\tilde{u}^{(t)}_{ij}\right)^{m}}$;
		\IF {$\|{\V^{(t+1)}-\V^{(t)}}\|\geq\varepsilon$}
			\STATE Set $t:=t+1$;\\
			\STATE Goto Step \ref{step3};
		\ELSE
        \RETURN $\V=\V^{(t+1)}$;
        \ENDIF
	\end{algorithmic}
\end{algorithm}

For simplicity, the flowchart of AMFCM is as follows:
\begin{equation*}
    \V^{(t)}\xrightarrow[\eqref{eq6},~\eqref{eq_modified
    u}]{\xrightarrow{\eqref{eq_3}}\U^{(t)} \xrightarrow{\eqref{eq_2}}\bar{\V}^{(t+1)}, \delta_{i}^{(t)}=\|\bar{\vv}_{i}^{(t+1)}-\vv_{i}^{(t)}\|}\tilde{\U}^{(t)}\xrightarrow{\eqref{eq_2}} \V^{(t+1)}.
  \end{equation*}

There are the following comments for AMFCM.
\begin{itemize}
  \item As shown above, the difference between AMFCM and FCM is the calculation of $\tilde{\U}^{(t)}$, which is the novelty of AMFCM. This inserted part can improve the efficiency and quality of the clustering algorithms.
  \item The extra cost of AMFCM over FCM is Step \ref{alg_lin0}-\ref{alg_lin3}. The cost in Step \ref{alg_lin0} is $\mathcal{O}(ncp)$, the cost of computing $\delta_{i},(1\le i\le c)$ in Step \ref{alg_lin1} is only $\mathcal{O}(cp)$, and the new cluster filtering technique \eqref{eq6} needs $\mathcal{O}(n(c-1))$, where $d_{ij}^{(t)}$ in \eqref{eq6} and \eqref{eq_modified u} has been calculated in Step \ref{step3}. Therefore, the cost of AMFCM is $\mathcal{O}(3ncp)$ per iteration (the cost of FCM is $\mathcal{O}(2ncp)$ per iteration).  
  \item The time complexity of FCM is $\mathcal{O}(nc^{2}pt_{\textbf{FCM}})$ \cite{kolen2002reducing, bhat2022reducing}, where $t_{\textbf{FCM}}$ is the iteration of FCM. For the time complexity of AMFCM, the update of the centers $\V_{\textbf{AMFCM}}$ requires $\mathcal{O}((\sum_{j=1}^{n} |\mathcal{P}_{j}|) p)$ per iteration, where  $1\le |\mathcal{P}_{j}| \le c$ for $j=1,2,...,n$. With the update of the membership grade matrix $\tilde{\U}$, the time complexity of AMFCM is  $\mathcal{O}(nc^{2}pt_{\textbf{AMFCM}})$, where $t_{\textbf{AMFCM}}$ is the iteration of AMFCM. Based on \textbf{Theorem} \ref{theorem3}, $t_{\textbf{AMFCM}} < t_{\textbf{FCM}}$. Aa a result, AMFCM can save the running time.
  
  \item For any membership grade of AMFCM, ${\widetilde{u}}_{ij}\in\left[0,1\right]$ for $i=1,2,...,c$ and $j=1,2,...,n$, based on Eq. \eqref{eq_modified u}. Obviously, AMFCM is a combination of the hard and soft clustering algorithms, which sets the membership grades to 0 for eliminating the contributions of the samples in the update of their non-affinity centers. Meanwhile, the fuzzy-type update way is performed for the remaining centers. In particular, AMFCM is a parameter-free and adaptive clustering algorithm, which can autonomously determine all samples-centers affinities and the update method.
\end{itemize}

\section{Theoretical Analysis} \label{sec6}
In this section, the convergence  properties of AMFCM are provided. First, the flowchart of FCM and AMFCM in $t$ iteration are as follows:
\begin{equation*}
\begin{aligned}
  \textbf{FCM}:~~& \V^{(t)}\xrightarrow{\eqref{eq_3}}\U_{\textbf{FCM}}^{(t)}\xrightarrow{\eqref{eq_2}} \V_{\textbf{FCM}}^{(t+1)}. \\
    \textbf{AMFCM}:~~& \V^{(t)}\xrightarrow[\delta_{i}^{(t)}=\| {\vv_{\textbf{FCM}}}_{i}^{(t+1)}-\vv_{i}^{(t)}\|, \eqref{eq6},~\eqref{eq_modified
    u}]{\xrightarrow{\eqref{eq_3}}\U_{\textbf{FCM}}^{(t)}\xrightarrow{\eqref{eq_2}}\V_{\textbf{FCM}}^{(t+1)}}\tilde{\U}^{(t)}\xrightarrow{\eqref{eq_2}} \V_{\textbf{AMFCM}}^{(t+1)}.\\
    \end{aligned}
  \end{equation*}
  
\begin{theorem}
The number of the iteration of AMFCM is smaller than that of FCM in the clustering process.
\label{theorem3}
\end{theorem}
\begin{proof}
The proof is given in \textbf{Appendix} \ref{sec7-4}. 
\end{proof}
  
 \begin{theorem}
AMFCM does not converge precociously in the mid stage of the clustering process.
\label{theorem1}
\end{theorem}
\begin{proof}
The proof is given in \textbf{Appendix} \ref{sec7-2}. 
\end{proof}

  \begin{theorem}
For the samples $\x_{j}$ with $|\mathcal{P}_{j}^{(t)}|=1$, the corresponding hard objective value of AMFCM is smaller than that of FCM in $t$ iteration.
\label{theorem2}
\end{theorem}
\begin{proof}
The proof is given in \textbf{Appendix} \ref{sec7-3}.
\end{proof}

In the next section, several experiments are performed to illustrate the efficiency of the proposed algorithm.

\section{Experimental Results} \label{sec5}
To verify the effectiveness and efficiency of the proposed algorithm, experimental studies are carried out on synthetic and real-world data sets, respectively. AMFCM is compared with another seven clustering algorithms, including:
\begin{enumerate}
  \item Fuzzy C-Means (FCM)  \cite{bezdek1984fcm},
  \item Rough Fuzzy C-Means (RFCM)  \cite{Mitra2006Rough},
  \item Shadowed Set-based Rough C-Means (SRCM), Shadowed Set-based Rough Fuzzy C-Means $\textrm{\uppercase\expandafter{\romannumeral1}}$ (SRFCM $\textrm{\uppercase\expandafter{\romannumeral1}}$), and Shadowed Set-based Rough Fuzzy C-Means $\textrm{\uppercase\expandafter{\romannumeral2}}$ (SRFCM $\textrm{\uppercase\expandafter{\romannumeral2}}$)  \cite{zhou2011Shadowed},
  \item Rough-Fuzzy Clustering based on Two-stage Three-way Approximations (ARFCM) \cite{Zhou2018Rough},
  \item Membership Scaling Fuzzy C-Means (MSFCM) \cite{Zhou2020A}.
\end{enumerate}

These algorithms are chosen because they use different techniques to reduce the contributions of the samples in the update of their non-affinity centers and increase the contributions of the samples in the update of the remaining centers for good clustering quality and fast convergence.

All experiments are run on a computer with an Intel Core i7-6700 processor and a maximum memory of 8GB for all processes; the computer runs Windows 7 with MATLAB R2017a. The experimental setup and the evaluation metrics used for clustering performance are first described. The fuzziness weighting exponent $m=2$ and the termination parameter $\varepsilon=10^{-6}$ for all algorithms. In addition, the weight exponent of the core region $w_{l}=0.95$ and the weight exponent of the boundary region $w_{b}=1-w_{l}$ for RFCM \cite{Mitra2006Rough}, SRCM, SRFCM $\textrm{\uppercase\expandafter{\romannumeral1}}$ and SRFCM
$\textrm{\uppercase\expandafter{\romannumeral2}}$ \cite{zhou2011Shadowed}, ARFCM \cite{Zhou2018Rough}.

\subsection{Evaluation Metrics}
In order to evaluate the performances of the newly proposed clustering algorithms, three external metrics, including the overall F-measure for the entire data set ($\textbf{F}^{*}$), Normalized Mutual Information (\textbf{NMI}), and Adjusted Rand Index (\textbf{ARI}) \cite{parker2013accelerating, mei2016large,Hubert1985Comparingpartitions}, are used. All three metrics are used to measure the agreement of the ground truth and the clustering results produced by an algorithm. The metrics that do not require the labels of data are also used for performance evaluation, called the internal metrics. The three internal validity metrics are selected, including \textbf{PC} \cite{James1973Cluster},  \textbf{DBI} \cite{Davies1979Cluster}, and \textbf{XB} \cite{xie1991validity}.
\begin{align}
\textbf{PC}&=\frac{1}{n}\sum_{i=1}^{c}\sum_{j=1}^{n} u_{ij}^{2},\\
\textbf{DBI}&=\frac{1}{c}\sum\limits_{k=1}^{c}\max_{i \neq k} {\frac{\frac{1}{|C_{i}|}\sum\limits_{\x_{j} \in C_{i}}
d_{ij}^{2}+\frac{1}{|C_{k}|}\sum\limits_{\x_{j}\in C_{k}} d_{kj}^{2}}{\|\vv_{i}-\vv_{k}\|^{2}}},\\
\textbf{XB}&=\frac{\sum_{i=1}^{c}\sum_{j=1}^{n} u_{ij}^{m}d_{ij}^{2}}{n\min_{i\neq k}\|\vv_{k}-\vv_{i}\|^{2}}.
\end{align}
Note that \textbf{Time} and \textbf{Iteration} are the remaining two evaluation metrics for expressing the efficiency of the algorithms.

\subsection{Experiments on Synthetic Data Sets}
In the first set of experiments, to test the efficiency of AMFCM in the whole convergence stages, two synthetic data sets in $\Real^{2}$ are implemented to observe the convergence path. The first synthetic data set contains three clusters, which are generalized by the two-dimensional Gaussian distribution with mean vector $\mu_{i}$ and covariance matrix $\Sigma_{i}$, $i=1,2,3$. The number of data in each cluster is 200, and the corresponding parameters $\mu_{i}$ and $\Sigma_{i}$ are $\mu_{1}=[10, 10]$, $\Sigma_{1}=\scriptsize{\setlength{\arraycolsep}{1.5pt}\begin{bmatrix} 0.3&0\\0&0.3\end{bmatrix}}$, $\mu_{2}=[13, 10]$, $\Sigma_{2}=\scriptsize{\setlength{\arraycolsep}{1.5pt}\begin{bmatrix} 0.8&0\\0&0.8\end{bmatrix}}$, and $\mu_{3}=[11, 4]$, $\Sigma_{3}=\scriptsize{\setlength{\arraycolsep}{1.5pt}\begin{bmatrix} 1.2&0\\0&1.2\end{bmatrix}}$, respectively. To further reflect the effect of the contributions of the samples in other clusters on the convergence of the algorithms, the second synthetic data set is designed, in which some samples are added to the first synthetic data set. These two synthetic data sets are called D1 and D2, respectively.

To illustrate intuitively, the visualized figures of the convergence trajectories of FCM, MSFCM, and AMFCM on D1 and D2 for comparison in this part, where same initializations are selected,  are shown as Fig. \ref{fig6}. \textbf{Time} and \textbf{Iterations} are selected to characterize the performance of the algorithms. Here, $t_{\textrm{A}}$ and $t_{\textrm{B}}$ are defined as the number of iteration of the algorithm in stages [A] and [B], respectively.
\begin{figure*}[t]
\centering
\subfloat[On D1]{\includegraphics[width=0.25\textwidth]{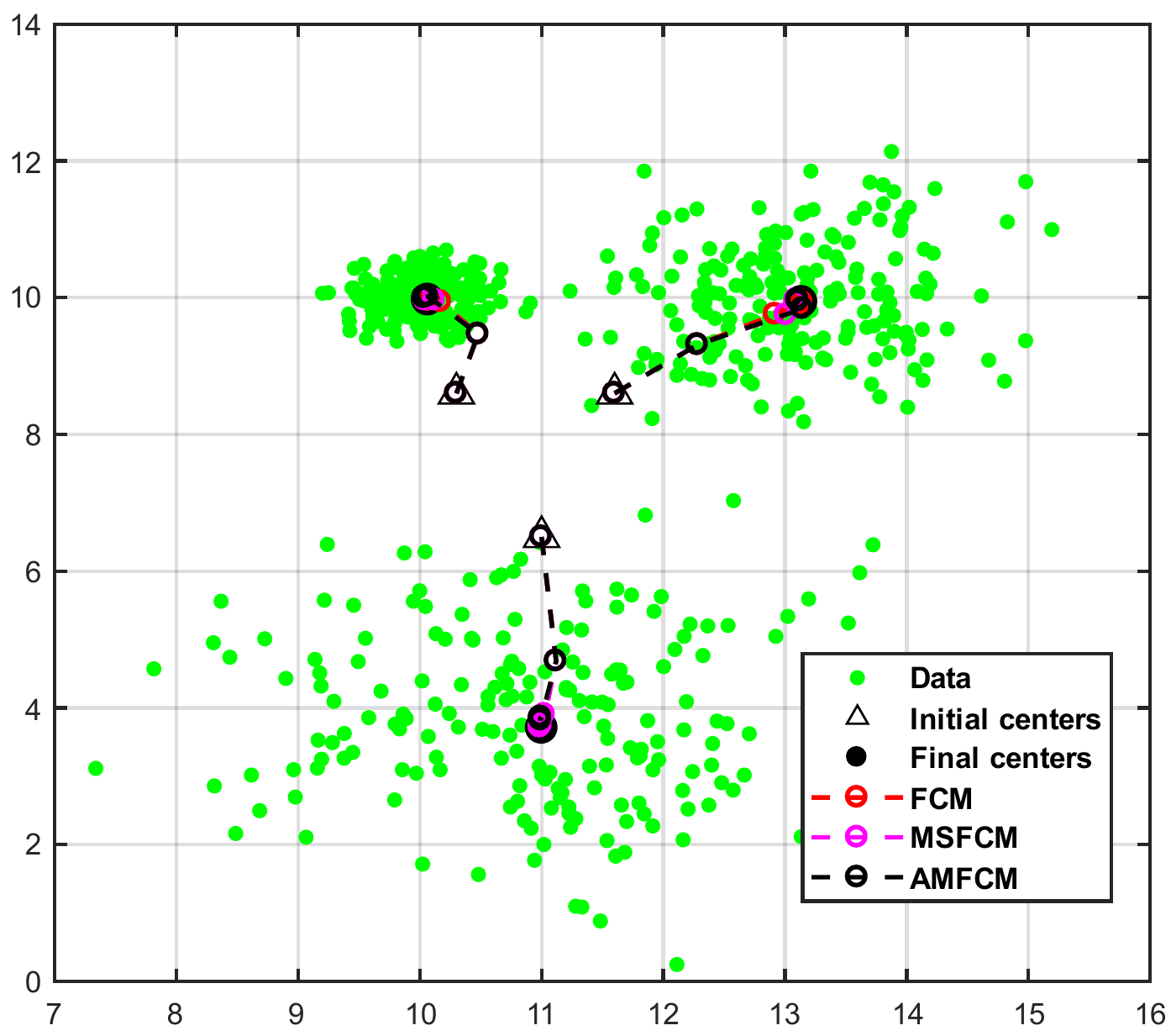}\label{fig6a}}
\subfloat[FCM on D2]{\includegraphics[width=0.25\textwidth]{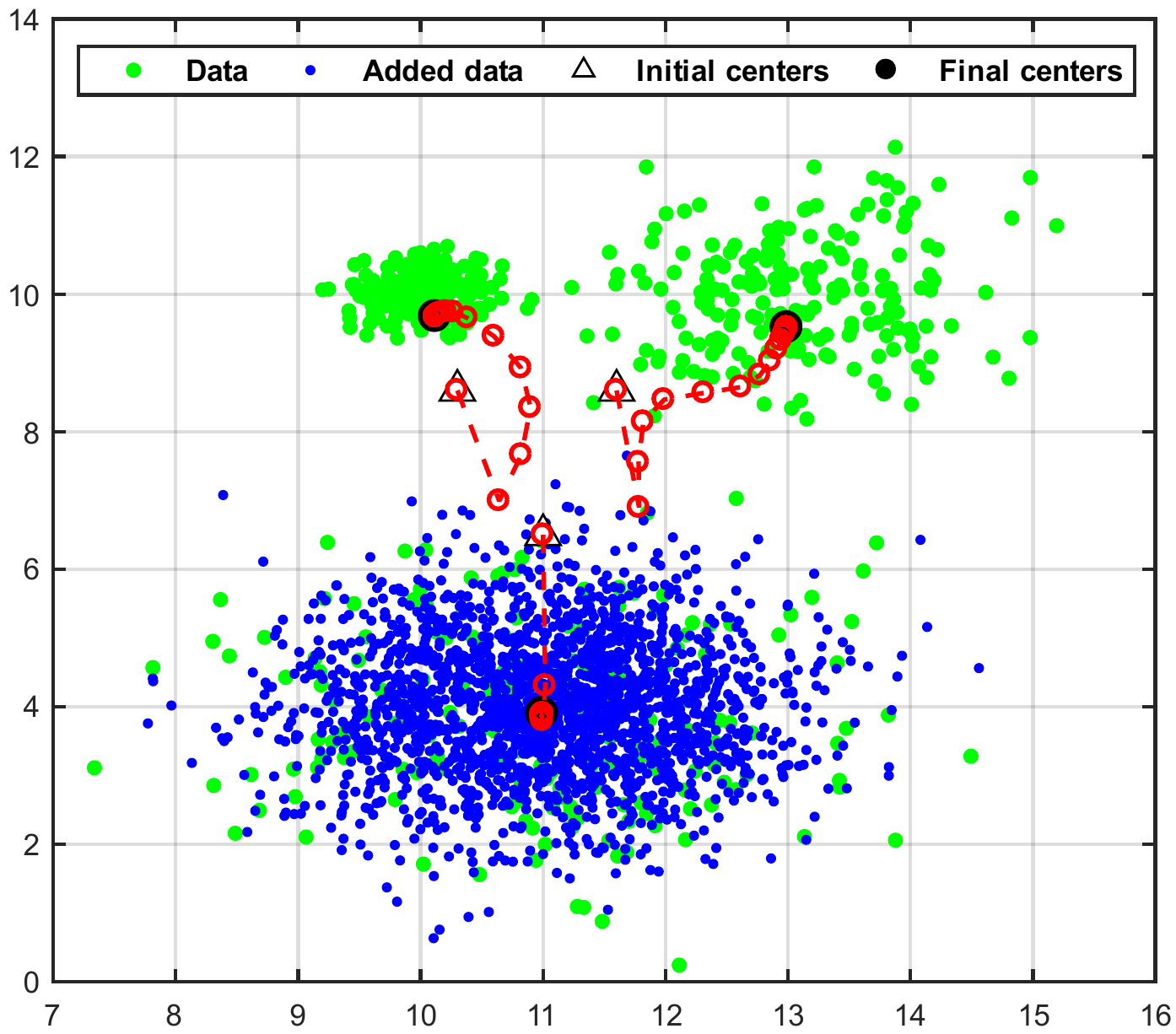}\label{fig6b}}
\subfloat[MSFCM on D2]{\includegraphics[width=0.25\textwidth]{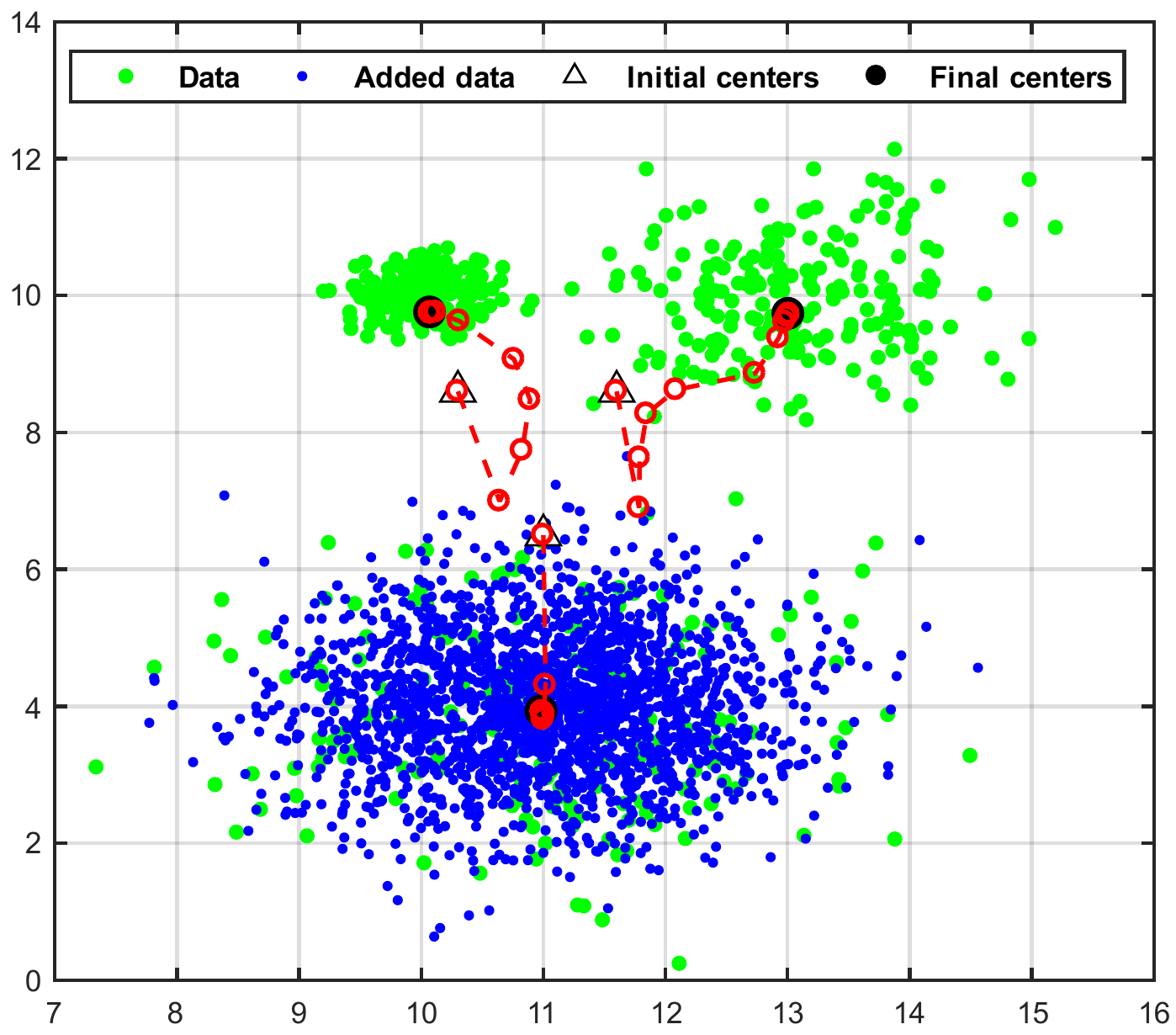}\label{fig6c}}
\subfloat[AMFCM on D2]{\includegraphics[width=0.25\textwidth]{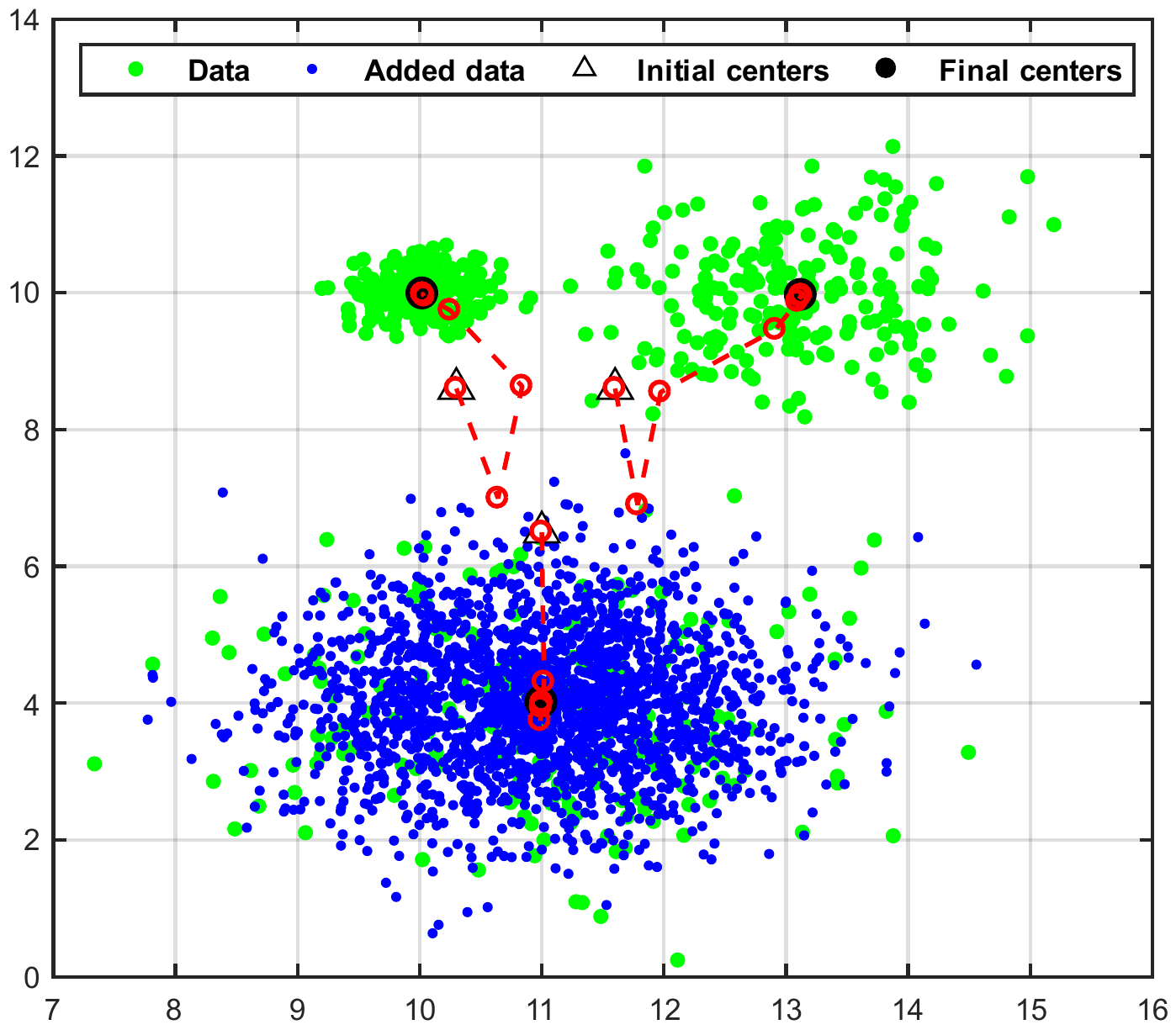}\label{fig6d}}
\caption{The convergence trajectories of FCM, MSFCM and AMFCM with same initializations on data sets D1 and D2. The
convergence trajectories of the three algorithms on D1 are put together, as shown in Fig. \ref{fig6a}.  The
convergence trajectories of the three algorithms on D2 are as shown in Fig. \ref{fig6b}, \ref{fig6c} and \ref{fig6d}, respectively. On D1, FCM, MSFCM and AMFCM are converged with 12 ($t_{\textrm{A}}$=2; $t_{\textrm{B}}$=10), 10 ($t_{\textrm{A}}$=2; $t_{\textrm{B}}$=8), and 8 ($t_{\textrm{A}}$=2; $t_{\textrm{B}}$=6) iterations and take 0.1294, 0.0781, and 0.054 seconds, respectively. On D2, FCM, MSFCM and AMFCM are converged with 30 ($t_{\textrm{A}}$=8; $t_{\textrm{B}}$=22), 18 ($t_{\textrm{A}}$=6; $t_{\textrm{B}}$=12), and 10 ($t_{\textrm{A}}$=3; $t_{\textrm{B}}$=7) iterations and take 0.2830, 0.1094, and 0.066 seconds, respectively.}\label{fig6}
\end{figure*}

First of all, the convergence trajectories of FCM, MSFCM, and AMFCM on D1 are similar, where the total contributions of the samples in the update of their non-affinity centers are small and not enough to observe. Therefore, the convergence trajectories of the three algorithms on D1 are put together, as shown in Fig. \ref{fig6a}. The experimental results on D1 show that AMFCM performs best, as shown in Fig. \ref{fig6}. From the convergence trajectories on D1, it is observed that although $t_{\textrm{A}}$=2 for FCM, MSFCM and AMFCM, $t_{\textrm{B}}$ of AMFCM is the least with only 6 iterations. On D2, the contributions of the new added samples in the update of their non-affinity centers mislead the displacement direction of the centers by the membership grades for FCM and MSFCM. Therefore, stages [A] and [B] of FCM and MSFCM  have been extended, where $t_{\textrm{A}}$=8 and $t_{\textrm{B}}$=22 for FCM; $t_{\textrm{A}}$=6 and $t_{\textrm{B}}$=12 for MSFCM. However, AMFCM completely eliminates the misleading of the total contributions of the newly added samples in the update of their non-affinity centers, which achieves better performance. From the convergence trajectory of AMFCM on D2, $t_{\textrm{A}}$=3 and $t_{\textrm{B}}$=7 for AMFCM.

Secondly, it is observed that the efficiency of MSFCM is higher than that of FCM. But the performance of MSFCM is limited. As analyzed in Section \ref{subsec3-1}, the previous affinity filtering \eqref{eq_4} of MSFCM fails to obtain the complete set of the non-affinity centers of each sample on D1 and D2, where $|\mathcal{P}|=1$ for some samples in this stage. Therefore, MSFCM does not always maintain high efficiency, because the membership scaling \eqref{eq_4} is invalid. However, AMFCM makes up for this shortcoming.

AMFCM accelerates the whole convergence process of FCM under same initializations. Both stages [A] and [B] have been accelerated. Especially, AMFCM can save 67\% of the number of the iteration of FCM on D2. As the previous complexity analysis, the running time of AMFCM is decreased. Thus, it can be concluded that the new affinity filtering scheme \eqref{eq6} is implemented with high efficiency, and the new membership scaling scheme \eqref{eq_modified u} is outstanding in terms of efficiency and clustering quality.
\subsection{Experiments on Real-World Data Sets}
In this subsection, some experiments are done to verify the clustering efficiency and performance of AMFCM  on real-world data sets.
\subsubsection{Acceleration and Performance of AMFCM}
To verify the acceleration of AMFCM in stages [A] and [B] on real-world data sets, the experiments in Fig. \ref{fig1} are redone by AMFCM with the same settings, and the corresponding hard-objective of AMFCM are also displayed, which is to illustrate the clustering characteristics of AMFCM,  as shown in Fig. \ref{fig7}.
\begin{figure}[htp!]
\centering
    \label{fig3a}\includegraphics[width=0.49\textwidth]{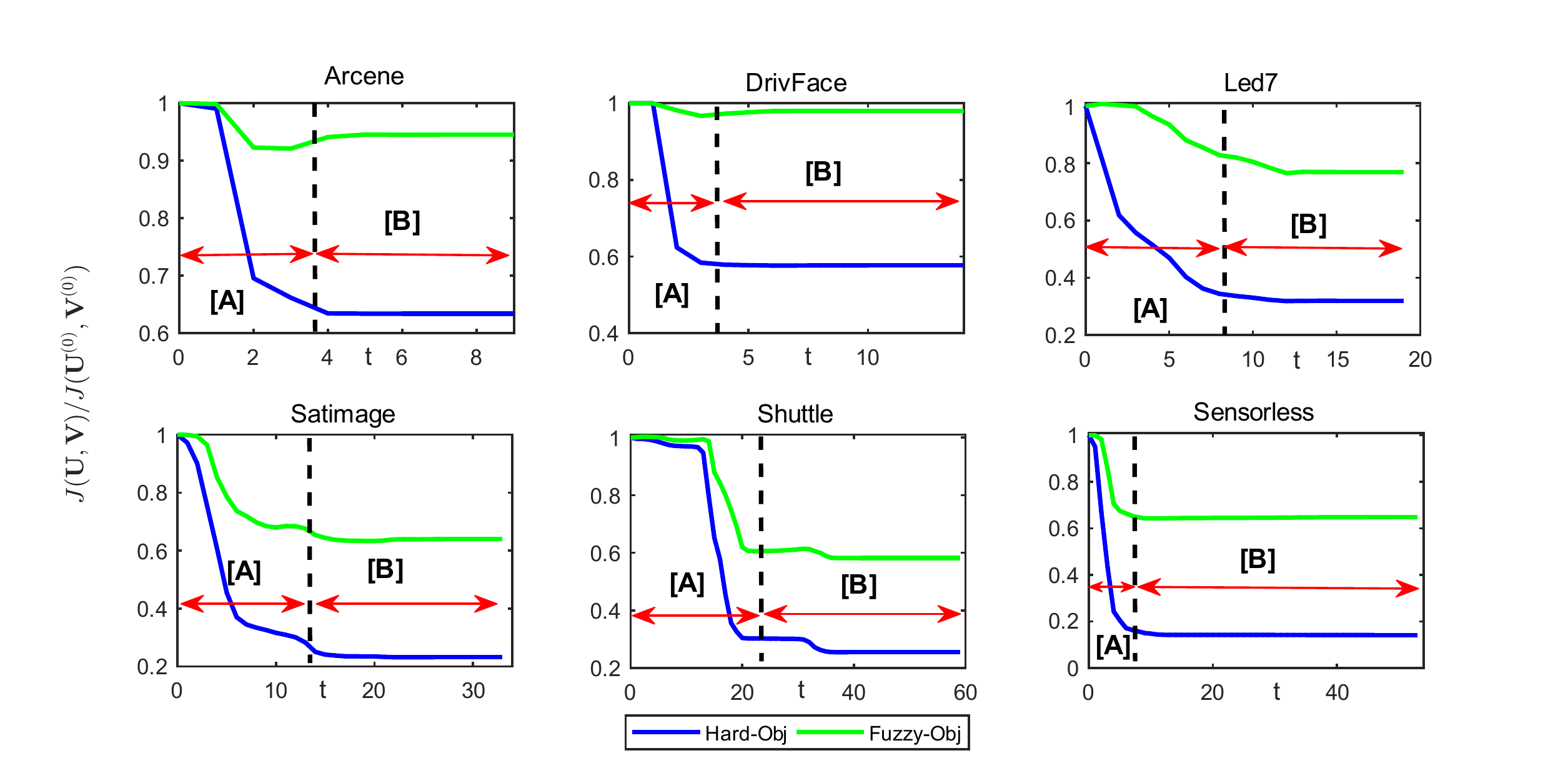}~~~~~~
\caption{Plots of $\frac{J_{\textbf{Fuzzy}}(\U^{(t)}, \V^{(t)})}{J_{\textbf{Fuzzy}}(\U^{(0)}, \V^{(0)})}$ and $\frac{J_{\textbf{Hard}}(\U^{(t)}, \V^{(t)})}{J_{\textbf{Hard}}(\U^{(0)}, \V^{(0)})}$ for iteration $t$ on six real-world data sets with AMFCM. The initializations is selected randomly for each data set. The plots clearly show that the clustering process of AMFCM can be divided into stages [A] and [B], where [A] represents the early stage and [B] represents the  mid-to-late stage.}\label{fig7}
\end{figure}

Similar to FCM, the fuzzy-objective and the corresponding hard-objective of AMFCM also can be divided into stages [A] and [B]. However, compared with Fig. \ref{fig1}, the number of the iteration of AMFCM is much lower than that of FCM, which saves at least 76$\%$ of the total rounds of the iteration on these six real-world data sets. Meanwhile, stages [A] and [B] of the convergence process of AMFCM are terminated earlier than the corresponding FCM, which is attributed to the new affinity filtering and membership scaling schemes, as shown in Fig. \ref{fig7}.

Recently, Nie \emph{et al.} \cite{Coordinate2021Nie} mentioned that bad local minimum makes the objective value not small enough, which limits the performance of the algorithms. According to this, the comparison results of the fuzzy and corresponding hard objective value of FCM and AMFCM on these six real-world data sets are displayed in TABLE \ref{table2}.
\begin{table}[h]
\centering
\caption{Comparison for the fuzzy and hard objective value with FCM and AMFCM. The values are averaged over 10 trials with random initializations. The best results are shown in boldface.}
\label{table2}
\begin{tabular}{r|rr|rr} \toprule
\multicolumn{1}{r|}{\multirow{2}{*}{Data sets}}
&\multicolumn{2}{c|}{\multirow{1}{*}{Fuzzy Objective Value}}
&\multicolumn{2}{c}{\multirow{1}{*}{Hard Objective Value}}
\\
\vspace{1.5mm}
&\multicolumn{1}{c}{\multirow{2}{*}{FCM}}
&\multicolumn{1}{c|}{\multirow{2}{*}{AMFCM}}
&\multicolumn{1}{c}{\multirow{2}{*}{FCM}}
&\multicolumn{1}{c}{\multirow{2}{*}{AMFCM}}
\\
\midrule
Arcene& \textbf{4.297$\texttt{E}$+04} &4.439$\texttt{E}$+04 & 6.635$\texttt{E}$+04 &\textbf{6.261$\texttt{E}$+04}\\
DrivFace&\textbf{6.079$\texttt{E}$+04} &6.362$\texttt{E}$+04 &1.315$\texttt{E}$+05 &\textbf{1.101$\texttt{E}$+05}\\
Led7&4.302$\texttt{E}$+02 &\textbf{3.709$\texttt{E}$+02}&3.468$\texttt{E}$+03 &\textbf{1.468$\texttt{E}$+03}\\
Satimage&\textbf{5.988$\texttt{E}$+02} &6.474$\texttt{E}$+02&1.764$\texttt{E}$+03 &\textbf{1.389$\texttt{E}$+03}\\
Shuttle&1.383$\texttt{E}$+02 & \textbf{1.326$\texttt{E}$+02}&2.925$\texttt{E}$+02& \textbf{2.588$\texttt{E}$+02}  \\
Sensorless&\textbf{1.869$\texttt{E}$+03} &1.919$\texttt{E}$+03&6.243$\texttt{E}$+03 &\textbf{4.558$\texttt{E}$+03}\\
\bottomrule
\end{tabular}
\end{table}

In TABLE \ref{table2}, the fuzzy objective value cannot achieve a small value because some values derived by the ordinary optimization theory are modified by AMFCM. The fuzzy objective is sacrificed for the efficiency of the algorithm. However, AMFCM can increase the membership of each sample to its nearest center through Eq. \ref{eq_modified u}, so that the corresponding hard objective of AMFCM is continuously optimized. Therefore, the corresponding hard objective value of AMFCM is better than that of FCM. Furthermore, AMFCM not only greatly improves the efficiency of FCM, but also maintains better clustering performance.

In order to display the performance of AMFCM more comprehensively, AMFCM is compared with the seven chosen clustering algorithms on the above eight evaluation metrics. Moreover, all real-world data sets, which are selected from UCI Machine Learning Repository\footnote{\url{https://archive.ics.uci.edu/ml/index.php}}, are clustered by the chosen clustering algorithms. The detailed information on the data sets is given in each table title, where $n$ is the number of training size, $p$ is the dimensionality of samples, and $c$ is the given number of clusters. The values are averaged over 10 trials with random initializations and the standard deviations are given after the means (linked with $\pm$), and the best results are shown in boldface. In addition, the corresponding \textbf{Iteration}  and \textbf{Time} of the different algorithms on ten real-world data sets are shown in Fig. \ref{fig9}.
\begin{figure*}[htp!]
\centering
\includegraphics[width=0.45\textwidth]{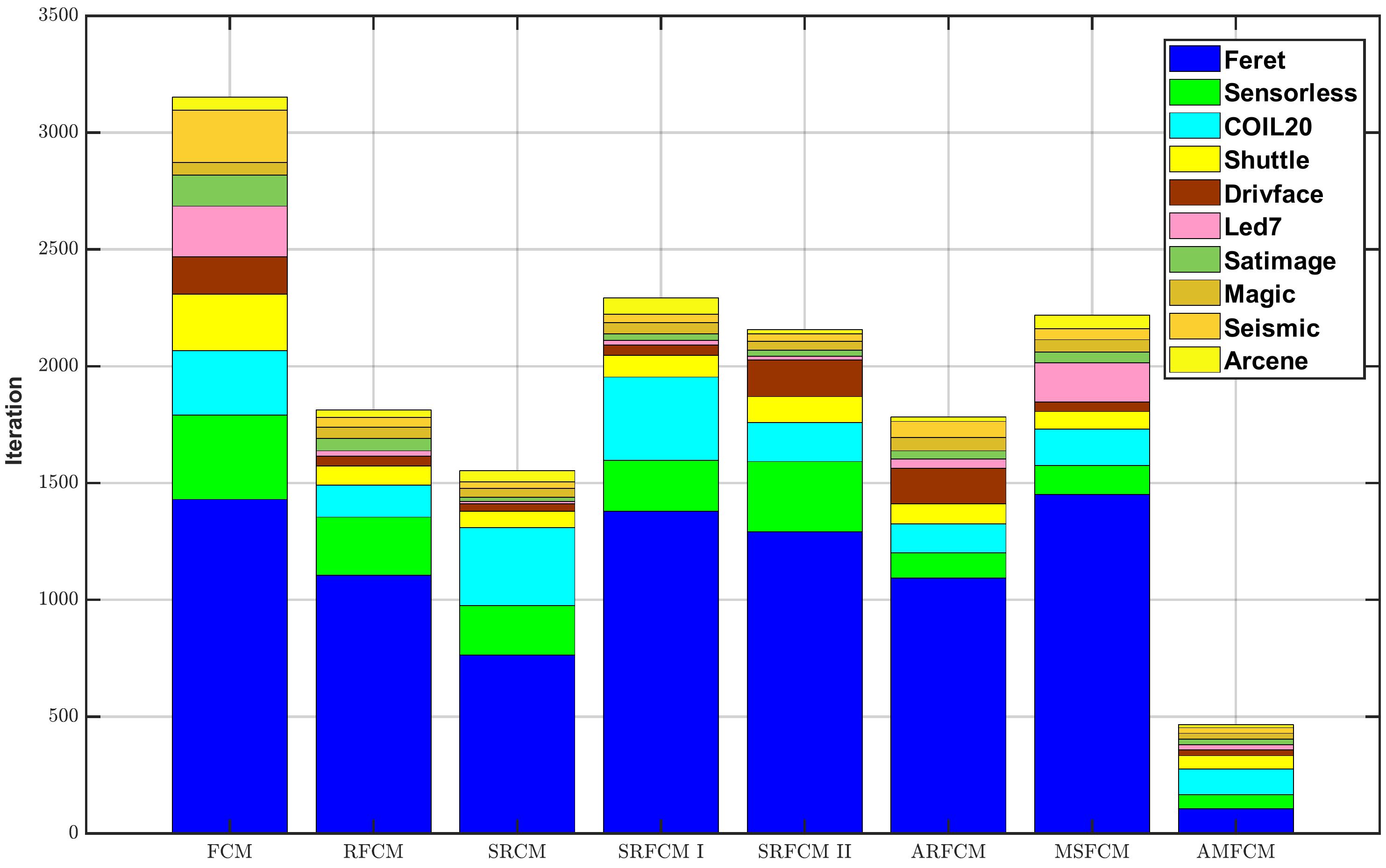}~~~~~~~~~~
\includegraphics[width=0.45\textwidth]{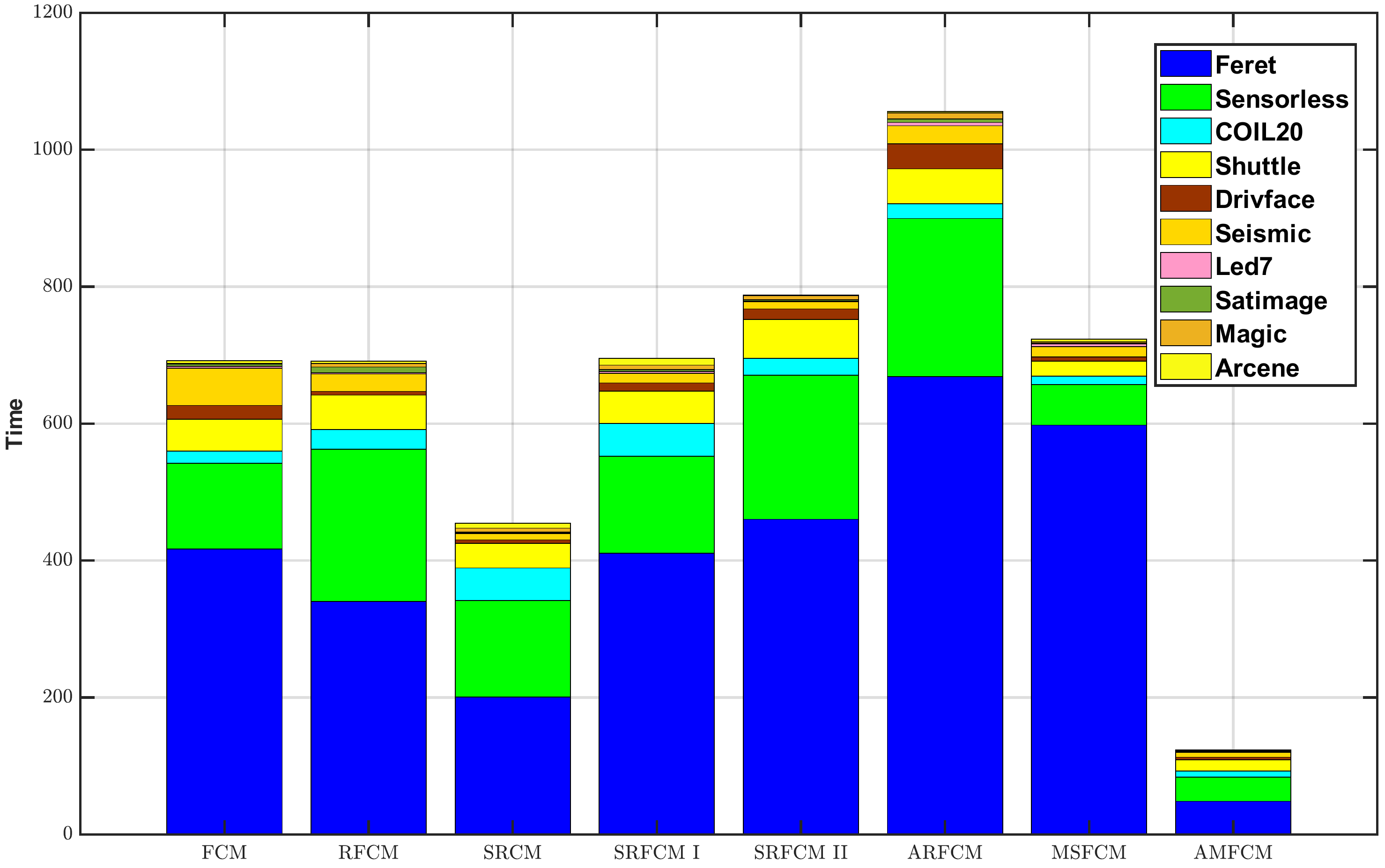}
\caption{Plot of the corresponding \textbf{Iteration}  and \textbf{Time} of the different algorithms on ten real-world data sets.}\label{fig9}
\end{figure*}

\begin{table*}[htp!]
\caption{Experimental results on ten real-world data sets for different algorithms. The values are averaged over 10 trials with random initializations. The standard deviations are given after the means (linked with $\pm$), and the best results are shown in boldface.}
\label{table3}
\centering
\resizebox{\textwidth}{!}{
\begin{threeparttable}
\begin{tabular}{l|r|rrrrrrrr} \toprule
\multicolumn{1}{l|}{\multirow{1}{*}{Data sets}}
&\multicolumn{1}{c|}{\multirow{1}{*}{Metrics \tnote{1}}}
&\multicolumn{1}{c}{\multirow{1}{*}{\textbf{FCM}}}
&\multicolumn{1}{c}{\multirow{1}{*}{\textbf{RFCM}}}
&\multicolumn{1}{c}{\multirow{1}{*}{\textbf{SRCM}}}
&\multicolumn{1}{c}{\multirow{1}{*}{\textbf{SRFCM $\textrm{\uppercase\expandafter{\romannumeral1}}$}}}
&\multicolumn{1}{c}{\multirow{1}{*}{\textbf{SRFCM $\textrm{\uppercase\expandafter{\romannumeral2}}$}}}
&\multicolumn{1}{c}{\multirow{1}{*}{\textbf{ARFCM}}}
&\multicolumn{1}{c}{\multirow{1}{*}{\textbf{MSFCM}}}
&\multicolumn{1}{c}{\multirow{1}{*}{\textbf{AMFCM}}}
\\
\midrule
&PC $\uparrow$&0.655$\pm$0.001&0.671$\pm$0.001&0.663$\pm$0.023&0.698$\pm$0.001&0.652$\pm$0.001&\textbf{0.710}$\pm$\textbf{0.001}&0.655$\pm$0.001&0.658$\pm$0.001\\
\textbf{Arcene}&DBI $\downarrow$&1.015$\pm$0.001&1.031$\pm$0.001&1.042$\pm$0.167&0.857$\pm$0.001&1.124$\pm$0.002&1.000$\pm$0.001&1.015$\pm$0.001&\textbf{0.831}$\pm$\textbf{0.001}\\
$n$=200&XB $\downarrow$&0.379$\pm$0.001&0.357$\pm$0.001&0.386$\pm$0.001&0.337$\pm$0.001&0.405$\pm$0.001&0.335$\pm$0.001&0.379$\pm$0.001&\textbf{0.316}$\pm$\textbf{0.001}\\
$p$=10000&$F^{*}$ $\uparrow$&0.583$\pm$0.001&0.641$\pm$0.001&0.633$\pm$0.004&0.586$\pm$0.001&0.643$\pm$0.001&0.649$\pm$0.001&0.583$\pm$0.001&\textbf{0.654}$\pm$\textbf{0.001}\\

$c$=2&ARI $\uparrow$&0.027$\pm$0.001&0.089$\pm$0.001&0.074$\pm$0.003&0.030$\pm$0.001&0.090$\pm$0.001&\textbf{0.091}$\pm$\textbf{0.001}&0.027$\pm$0.001&\textbf{0.091}$\pm$\textbf{0.001}\\

&NMI $\uparrow$&0.018$\pm$0.001&0.091$\pm$0.001&0.066$\pm$0.001&0.020$\pm$0.001&0.086$\pm$0.001&0.085$\pm$0.001&0.018$\pm$0.001&\textbf{0.087}$\pm$\textbf{0.001}\\

\midrule
&PC $\uparrow$&0.420$\pm$0.001&0.489$\pm$0.001&0.479$\pm$0.010&0.487$\pm$0.006&0.477$\pm$0.009&0.445$\pm$0.001&0.461$\pm$0.001&\textbf{0.490}$\pm$\textbf{0.001}\\

\textbf{DrivFace}&DBI $\downarrow$&15.7$\pm$0.1&2.641$\pm$0.011&3.638$\pm$1.375&2.258$\pm$.0461&4.964$\pm$1.809&2.666$\pm$0.810&2.677$\pm$0.001&\textbf{1.644}$\pm$\textbf{0.001}\\

$n$=606&XB $\downarrow$&5.354$\pm$0.001&0.897$\pm$0.004&1.685$\pm$0.425&1.187$\pm$0.267&1.732$\pm$0.601&0.936$\pm$0.072&0.902$\pm$0.001&\textbf{0.567}$\pm$\textbf{0.001}\\
$p$=6400&$F^{*}$ $\uparrow$&0.558$\pm$0.001&0.576$\pm$0.005&0.541$\pm$0.001&0.576$\pm$0.004&0.571$\pm$0.003&0.566$\pm$0.001&0.587$\pm$0.001&\textbf{0.597}$\pm$\textbf{0.001}\\

$c$=3&ARI $\uparrow$&0.016$\pm$0.001&0.021$\pm$0.001&0.008$\pm$0.001&0.006$\pm$0.001&0.020$\pm$0.003&0.019$\pm$0.001&0.019$\pm$0.001&\textbf{0.024}$\pm$\textbf{0.001}\\

&NMI $\uparrow$&0.053$\pm$0.001&0.054$\pm$0.001&0.030$\pm$0.002&0.035$\pm$0.007&0.045$\pm$0.002&0.054$\pm$0.001&0.053$\pm$0.001&\textbf{0.057}$\pm$\textbf{0.001}\\

\midrule
&PC $\uparrow$&\textbf{0.007}$\pm$\textbf{0.001}&\textbf{0.007}$\pm$\textbf{0.001}&0.006$\pm$0.003&\textbf{0.007}$\pm$\textbf{0.001}&0.007$\pm$0.002&0.005$\pm$0.001&\textbf{0.007}$\pm$\textbf{0.001}&\textbf{0.007}$\pm$\textbf{0.001}\\

\textbf{Feret}&DBI $\downarrow$&2.729$\pm$0.106&2.630$\pm$0.186&2.832$\pm$0.389&2.311$\pm$0.834&7.267$\pm$3.920&1.963$\pm$0.345&1.729$\pm$0.259&\textbf{0.927}$\pm$0.122\\

$n$=1400&XB $\downarrow$&0.127$\pm$0.001&0.292$\pm$0.001&0.156$\pm$0.041&0.062$\pm$0.010&5.984$\pm$4.184&0.056$\pm$0.001&0.088$\pm$0.001&\textbf{0.053}$\pm$0.009\\
$p$=1600&$F^{*}$ $\uparrow$&0.133$\pm$0.005&0.119$\pm$0.005&0.144$\pm$0.026&0.146$\pm$0.012&0.143$\pm$0.022&0.150$\pm$0.025&0.119$\pm$0.030&\textbf{0.153}$\pm$\textbf{0.002}\\

$c$=200&ARI $\uparrow$&0.016$\pm$0.034&0.016$\pm$0.020&0.022$\pm$0.009&0.024$\pm$0.011&0.019$\pm$0.005&0.021$\pm$0.006&0.021$\pm$0.003&\textbf{0.025}$\pm$\textbf{0.003}\\

&NMI $\uparrow$&0.415$\pm$0.009&0.382$\pm$0.015&0.415$\pm$0.028&0.419$\pm$0.021&0.402$\pm$0.019&0.425$\pm$0.017&0.419$\pm$0.058&\textbf{0.478}$\pm$0.034\\

\midrule
&PC $\uparrow$&0.055$\pm$0.001&0.056$\pm$0.006&0.091$\pm$0.009&0.097$\pm$0.009&0.102$\pm$0.009&0.055$\pm$0.001&0.058$\pm$0.005&\textbf{0.108}$\pm$0.018\\

\textbf{COIL20}&DBI $\downarrow$&1.776$\pm$0.225&2.961$\pm$0.154&42.3$\pm$21.2&42.3$\pm$1.7&37.9$\pm$15.6&1.276$\pm$0.074&1.787$\pm$0.001&\textbf{0.847}$\pm$0.012\\

$n$=1440&XB $\downarrow$&1.077$\pm$0.001&1.099$\pm$0.014&9.7$\pm$3.7&10.1$\pm$3.0&14.2$\pm$6.2&0.297$\pm$0.001&0.970$\pm$0.003&\textbf{0.127}$\pm$0.005\\

$p$=1024&$F^{*}$ $\uparrow$&0.242$\pm$0.028&0.228$\pm$0.001&0.282$\pm$0.041&0.401$\pm$0.032&0.397$\pm$0.010&0.271$\pm$0.001&0.260$\pm$0.015&\textbf{0.443}$\pm$0.060\\

$c$=20&ARI $\uparrow$&0.110$\pm$0.036&0.133$\pm$0.001&0.248$\pm$0.058&0.237$\pm$0.026&0.254$\pm$0.017&0.279$\pm$0.002&0.110$\pm$0.023&\textbf{0.286}$\pm$0.068\\

&NMI $\uparrow$&0.299$\pm$0.042&0.372$\pm$0.001&0.339$\pm$0.045&0.385$\pm$0.020&0.317$\pm$0.004&0.301$\pm$0.003&0.374$\pm$0.017&\textbf{0.574}$\pm$0.059\\

\midrule
&PC $\uparrow$&0.229$\pm$0.027&0.521$\pm$0.021&0.281$\pm$0.043&0.421$\pm$0.064&0.371$\pm$0.070&\textbf{0.593}$\pm$0.026&0.351$\pm$0.221&0.509$\pm$0.051\\

\textbf{Led7}&DBI $\downarrow$&0.986$\pm$0.131&1.145$\pm$0.100&2.852$\pm$0.575&1.454$\pm$0.229&1.807$\pm$0.229&1.160$\pm$0.206&0.983$\pm$0.117&\textbf{0.857}$\pm$0.105\\

$n$=3200&XB $\downarrow$&0.178$\pm$0.011&0.132$\pm$0.008&0.191$\pm$0.015&0.165$\pm$0.022&0.223$\pm$0.042&0.144$\pm$0.027&0.164$\pm$0.032&\textbf{0.122}$\pm$\textbf{0.002}\\

$p$=7&$F^{*}$ $\uparrow$&0.424$\pm$0.001&0.614$\pm$0.049&0.591$\pm$0.055&0.618$\pm$0.111&0.586$\pm$0.018&0.694$\pm$0.025&0.487$\pm$0.165&\textbf{0.731}$\pm$0.007\\

$c$=10&ARI $\uparrow$&0.232$\pm$0.001&0.415$\pm$0.043&0.376$\pm$0.063&0.405$\pm$0.101&0.356$\pm$0.034&0.438$\pm$0.021&0.436$\pm$0.152&\textbf{0.497}$\pm$0.007\\

&NMI $\uparrow$&0.366$\pm$0.001&0.498$\pm$0.038&0.465$\pm$0.052&0.494$\pm$0.077&0.466$\pm$0.033&0.510$\pm$0.016&0.473$\pm$0.104&\textbf{0.563}$\pm$0.019\\

\midrule
&PC $\uparrow$&0.390$\pm$0.001&0.432$\pm$0.029&0.315$\pm$0.045&0.351$\pm$0.010&0.341$\pm$0.038&0.408$\pm$0.019&0.448$\pm$0.028&\textbf{0.476}$\pm$\textbf{0.001}\\

\textbf{Satimage}&DBI $\downarrow$&4.880$\pm$0.001&2.464$\pm$0.862&39.4$\pm$12.4&8.906$\pm$3.432&28.8$\pm$10.9&3.681$\pm$1.310&2.261$\pm$1.048&\textbf{0.908}$\pm$\textbf{0.001}\\

$n$=6435&XB $\downarrow$&3.478$\pm$0.001&1.112$\pm$0.023&25.2$\pm$3.1&1.612$\pm$0.419&8.032$\pm$2.187&1.180$\pm$0.186&0.493$\pm$0.036&\textbf{0.455}$\pm$\textbf{0.001}\\

$p$=36&$F^{*}$ $\uparrow$&0.553$\pm$0.001&0.593$\pm$0.029&0.631$\pm$0.049&0.638$\pm$0.062&0.612$\pm$0.080&0.635$\pm$0.009&0.638$\pm$0.013&\textbf{0.659}$\pm$\textbf{0.001}\\

$c$=6&ARI $\uparrow$&0.292$\pm$0.001&0.317$\pm$0.027&0.348$\pm$0.051&0.358$\pm$0.071&0.337$\pm$0.018&0.389$\pm$0.005&0.406$\pm$0.019&\textbf{0.443}$\pm$\textbf{0.001}\\

&NMI $\uparrow$&0.450$\pm$0.001&0.461$\pm$0.023&0.457$\pm$0.029&0.457$\pm$0.034&0.446$\pm$0.062&0.471$\pm$0.007&0.493$\pm$0.006&\textbf{0.515}$\pm$\textbf{0.001}\\

\midrule
&PC $\uparrow$&0.656$\pm$0.001&0.720$\pm$0.001&0.733$\pm$0.001&0.717$\pm$0.001&0.702$\pm$0.001&0.679$\pm$0.036&0.656$\pm$0.001&\textbf{0.729}$\pm$\textbf{0.001}\\

\textbf{Magic}&DBI $\downarrow$&1.886$\pm$0.001&1.549$\pm$0.001&1.542$\pm$0.001&1.539$\pm$0.001&1.539$\pm$0.001&1.490$\pm$0.005&1.886$\pm$0.001&\textbf{1.050}$\pm$\textbf{0.001}\\

$n$=19020&XB $\downarrow$&0.545$\pm$0.001&0.390$\pm$0.001&0.373$\pm$0.001&0.381$\pm$0.001&0.372$\pm$0.001&0.438$\pm$0.010&0.545$\pm$0.001&\textbf{0.266}$\pm$\textbf{0.001}\\
$p$=10&$F^{*}$ $\uparrow$&0.582$\pm$0.001&0.627$\pm$0.001&0.622$\pm$0.001&0.633$\pm$0.001&0.612$\pm$0.001&0.591$\pm$0.003&0.582$\pm$0.001&\textbf{0.641}$\pm$\textbf{0.001}\\

$c$=2&ARI $\uparrow$&0.007$\pm$0.001&0.013$\pm$0.001&0.018$\pm$0.001&0.014$\pm$0.001&0.018$\pm$0.001&0.008$\pm$0.001&0.007$\pm$0.001&\textbf{0.020}$\pm$\textbf{0.001}\\

&NMI $\uparrow$&0.020$\pm$0.001&0.043$\pm$0.001&0.051$\pm$0.001&0.046$\pm$0.001&0.053$\pm$0.001&0.053$\pm$0.001&0.020$\pm$0.001&\textbf{0.057}$\pm$\textbf{0.001}\\

\midrule
&PC $\uparrow$&0.362$\pm$0.001&0.447$\pm$0.035&0.353$\pm$0.006&0.398$\pm$0.017&0.349$\pm$0.032&0.330$\pm$0.032&0.409$\pm$0.006&\textbf{0.513}$\pm$0.049\\

\textbf{Shuttle}&DBI $\downarrow$&290.5$\pm$0.8&157.8$\pm$27.4&307.7$\pm$124.3&230.4$\pm$30.2&244.6$\pm$27.1&287.9$\pm$63.9&153.6$\pm$19.4&\textbf{52.2}$\pm$33.7\\

$n$=58000&XB $\downarrow$&98.7$\pm$1.6&9.7$\pm$0.4&13.5$\pm$1.3&13.0$\pm$3.9&12.2$\pm$0.4&65.1$\pm$8.41&15.4$\pm$0.1&\textbf{7.8}$\pm$3.6\\

$p$=9&$F^{*}$ $\uparrow$&0.504$\pm$0.001&0.593$\pm$0.010&0.546$\pm$0.061&0.578$\pm$0.062&0.571$\pm$0.010&0.460$\pm$0.056&0.512$\pm$0.043&\textbf{0.667}$\pm$0.071\\

$c$=7&ARI $\uparrow$&0.114$\pm$0.001&0.157$\pm$0.026&0.117$\pm$0.054&0.157$\pm$0.073&0.149$\pm$0.007&0.085$\pm$0.022&0.153$\pm$0.052&\textbf{0.190}$\pm$0.062\\

&NMI $\uparrow$&0.218$\pm$0.001&0.257$\pm$0.027&0.201$\pm$0.066&0.226$\pm$0.023&0.206$\pm$0.030&0.212$\pm$0.036&0.218$\pm$0.060&\textbf{0.260}$\pm$0.016\\

\midrule
&PC $\uparrow$&0.263$\pm$0.016&0.251$\pm$0.023&0.202$\pm$0.026&0.191$\pm$0.024&0.112$\pm$0.004&0.241$\pm$0.016&0.295$\pm$0.023&\textbf{0.309}$\pm$\textbf{0.004}\\

\textbf{Sensorless}&DBI $\downarrow$&3.206$\pm$0.091&5.304$\pm$0.951&156.7$\pm$32.4&31.1$\pm$28.6&21.6$\pm$9.4&5.968$\pm$3.370&1.869$\pm$0.507&\textbf{0.814}$\pm$\textbf{0.054}\\

$n$=58509&XB $\downarrow$&1.494$\pm$0.067&3.051$\pm$0.667&35.5$\pm$14.7&25.2$\pm$14.5&15.9$\pm$8.6&3.906$\pm$1.513&1.195$\pm$0.635&\textbf{0.326}$\pm$\textbf{0.052}\\

$p$=48&$F^{*}$ $\uparrow$&0.307$\pm$0.014&0.282$\pm$0.021&0.281$\pm$0.017&0.271$\pm$0.020&0.288$\pm$0.023&0.275$\pm$0.028&0.311$\pm$0.023&\textbf{0.325}$\pm$\textbf{0.005}\\

$c$=11&ARI $\uparrow$&0.142$\pm$0.007&0.122$\pm$0.014&0.106$\pm$0.013&0.009$\pm$0.001&0.103$\pm$0.018&0.100$\pm$0.018&0.143$\pm$0.023&\textbf{0.147}$\pm$0.001\\

&NMI $\uparrow$&0.306$\pm$0.006&0.278$\pm$0.031&0.265$\pm$0.012&0.243$\pm$0.025&0.276$\pm$0.039&0.238$\pm$0.027&0.325$\pm$0.020&\textbf{0.339}$\pm$0.006\\

\midrule
&PC $\uparrow$&0.420$\pm$0.001&0.521$\pm$0.001&0.522$\pm$0.033&0.512$\pm$0.053&0.514$\pm$0.028&\textbf{0.544}$\pm$\textbf{0.001}&0.447$\pm$0.001&0.503$\pm$0.001\\

\textbf{Seismic}&DBI $\downarrow$&10.9$\pm$0.1&2.834$\pm$0.001&3.055$\pm$1.179&2.569$\pm$0.490&3.186$\pm$1.057&\textbf{2.008}$\pm$\textbf{0.001}&7.034$\pm$0.001&3.364$\pm$0.005\\

$n$=78823&XB $\downarrow$&2.220$\pm$0.001&0.566$\pm$0.001&0.724$\pm$0.356&0.474$\pm$0.027&0.734$\pm$0.336&\textbf{0.428}$\pm$\textbf{0.001}&1.412$\pm$0.001&0.680$\pm$0.001\\

$p$=30&$F^{*}$ $\uparrow$&0.448$\pm$0.001&0.460$\pm$0.001&0.493$\pm$0.006&0.492$\pm$0.011&\textbf{0.503}$\pm$0.006&0.491$\pm$0.001&0.451$\pm$0.001&0.471$\pm$0.001\\

$c$=3&ARI $\uparrow$&0.038$\pm$0.001&0.040$\pm$0.001&\textbf{0.069}$\pm$0.001&0.046$\pm$0.004&0.055$\pm$0.022&0.063$\pm$0.001&0.038$\pm$0.001&0.045$\pm$0.001\\

&NMI $\uparrow$&0.043$\pm$0.001&0.045$\pm$0.001&0.066$\pm$0.003&0.065$\pm$0.005&\textbf{0.081}$\pm$0.024&0.058$\pm$0.001&0.043$\pm$0.001&0.049$\pm$0.001\\

\bottomrule
\end{tabular}
\begin{tablenotes}
\item[1] The superscript '$\uparrow$' sign of evaluation metrics represents that the larger evaluation metrics, the better the clustering performance. The superscript '$\downarrow$' sign of evaluation metrics represents that the smaller evaluation metrics, the better the clustering performance.
\end{tablenotes}
\end{threeparttable}}
\end{table*}

From the experimental results in TABLE \ref{table3} and Fig. \ref{fig9}, the following conclusions are obtained.
\begin{itemize}
  \item Comparing the first and last columns of each data set, AMFCM has better performance than FCM on all data sets in terms of the above six evaluation metrics in TABLE \ref{table3}. Moreover, it is worth mentioning that the efficiency of AMFCM has been improved in the whole stages, which reduces the number of the iteration of FCM by 80$\%$ on average without significant computational cost in Fig. \ref{fig9}. Therefore, AMFCM has also achieved significant savings in running time. As shown in Fig. \ref{fig9}, the total \textbf{Iteration} and \textbf{Time} of AMFCM is much lower than that of other algorithms.
  \item According to \cite{Zhou2020A}, it is found that the cost of AMFCM and MSFCM are both $\mathcal{O}(3ncp)$ per iteration. For the experimental results of MSFCM in the penultimate column of each table, the acceleration of MSFCM on the two data sets failed because the affinity filtering\eqref{eq_4} fails to obtain the complete non-affinity information as analyzed in Section \ref{subsec3-1}, and the low efficiency of the membership scaling \eqref{eq_new u}, which will be explained in the next Subsection \ref{subsubsec5-3-2}. In the remaining eight data sets, although MSFCM is effective, the acceleration performance of AMFCM in the whole stages is better than that of MSFCM. Thus, AMFCM is a successful generalization to MSFCM.
  \item The remaining five algorithms sometimes have better clustering quality than AMFCM, which is due to the good parameters for the division of each cluster. In this case, the clustering efficiency of the algorithms is reduced, as shown in Fig. \ref{fig9}. However, searching the complete set of the non-affinity centers is a parameter-free and autonomous process for AMFCM. In summary, AMFCM is a good trade-off between  efficiency and quality.
\end{itemize}

According to the summary above, AMFCM can greatly improve clustering efficiency and quality on real-world data sets, which is based on the new affinity filtering \eqref{eq6} and membership scaling \eqref{eq_modified u} schemes. The acquisition and elimination of the redundant contributions of the samples in the update of the centers is the key to the success of AMFCM.

\subsubsection{Statistical Comparisons by Friedman Test}\label{subsubsec5-3-2}
In order to compare the  multiple algorithms systematically, the Friedman test \cite{Demiar2006Statistical} is applied to compare the clustering efficiency (\textbf{Time} and \textbf{Iteration}) and quality ($\textbf{F}^{*}$ and \textbf{ARI}) of the eight algorithms on the selected ten data sets. In detail, Friedman test at significance level $\alpha=0.05$ rejects the null hypothesis of equal performance, which leads to the use of post-hoc tests to find out which algorithms are actually different. Next, Nemenyi test is used to where the performance of two algorithms is significantly different if their average ranks over all datasets differ by at least one critical difference. The critical difference is defined as $\text{CD}=q_{\alpha} \sqrt{\frac{K(K+1)}{6N}}$, where critical values $q_{\alpha}$ are based on the Studentized range statistic divided by $\sqrt{2}$, $K$ is the number of the comparison algorithms, and $N$ is the number of the data sets. In this part, $\textbf{F}^{*}$ and \textbf{ARI} are selected to evaluate the clustering quality, and the remaining metric have the similar results. The critical difference (CD) diagrams, as shown in Fig. \ref{fig11}, are presented to analyze the significance between AMFCM and the comparison algorithms on the ten data sets with $\textbf{F}^{*}$, \textbf{ARI}, \textbf{Iteration} and \textbf{Time}, where the average rank of each algorithm is marked on the line and the axis. The axis is turned so that the lowest (best) ranks are to the right. Groups of algorithms that are not significantly different according to Nemenyi test are connected with a red line. The critical difference (CD = 3.3203 at 0.05 significance level) is also shown above the axis in each subfigure.

\begin{figure}[htp]
 \centering
 \subfloat[$\textbf{F}^{*}$]{\label{fig11a}\includegraphics[width=0.5\textwidth]{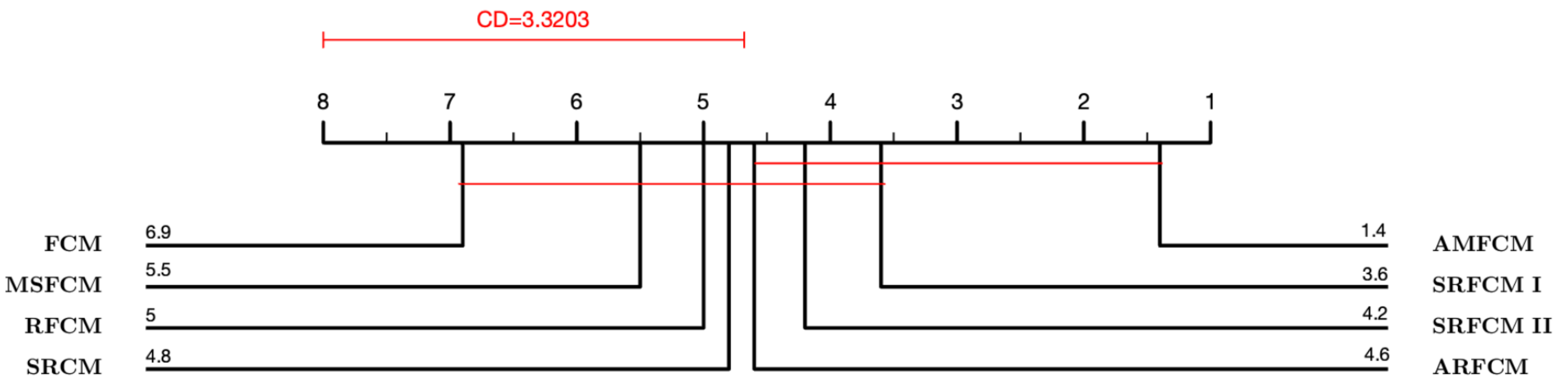}}\\
 \subfloat[$\textbf{ARI}$]{\label{fig11b}\includegraphics[width=0.5\textwidth]{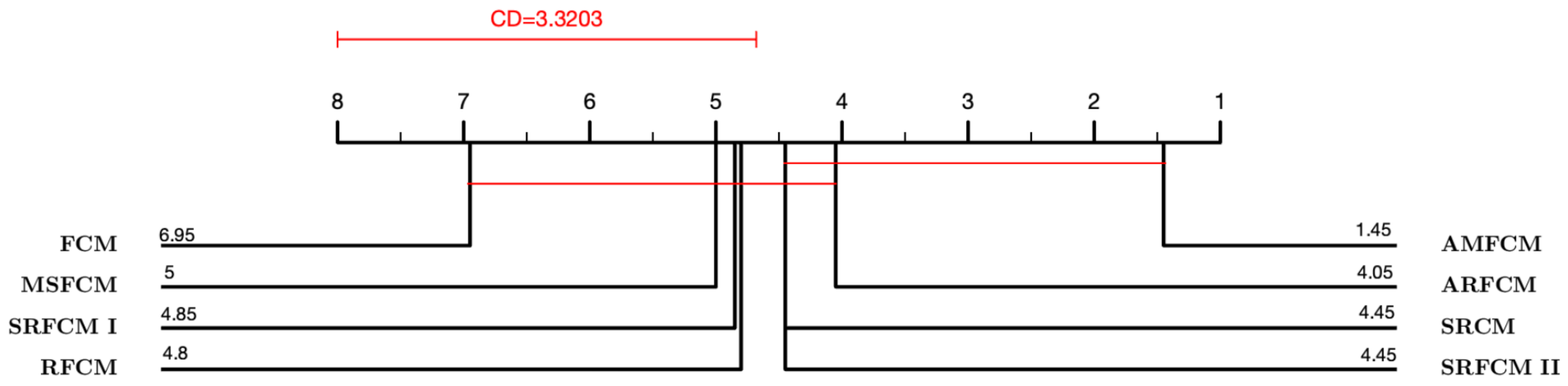}}\\
 \subfloat[$\textbf{Iteration}$]{\label{fig11c}\includegraphics[width=0.5\textwidth]{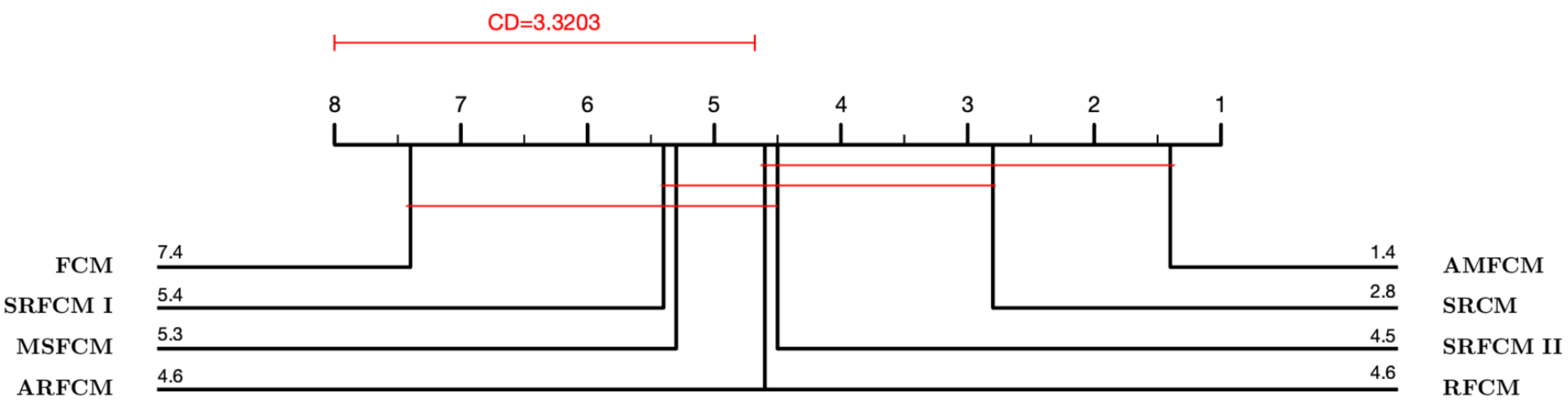}}\\
 \subfloat[$\textbf{Time}$]{\label{fig11d}\includegraphics[width=0.5\textwidth]{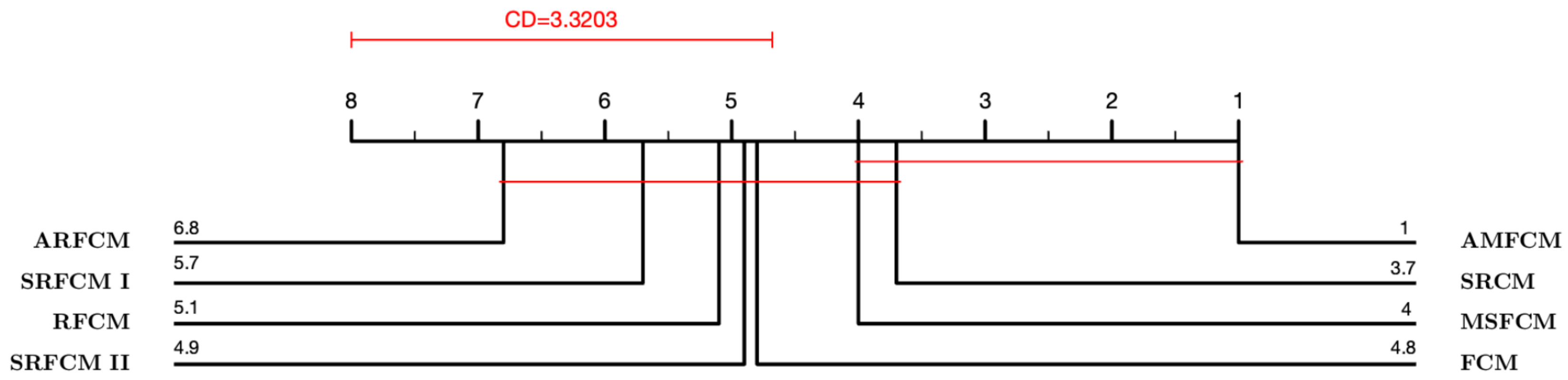}}
 \caption{CD diagrams of the eight comparison algorithms on the ten data sets with $\textbf{F}^{*}$, $\textbf{ARI}$, \textbf{Iteration}, and \textbf{Time}. It is clear that AMFCM statistically achieves a good trade-off between the clustering quality and efficiency.}\label{fig11}
 \end{figure}

According to the CD diagrams, first, in the clustering efficiency and quality, AMFCM achieves the statistically superior performance than that of FCM on the ten data sets. Second, from the Fig. \ref{fig11a} and \ref{fig11b}, AMFCM presents statistically comparable clustering quality with ARFCM on the ten data sets. However, AMFCM  statistically outperforms ARFCM in the clustering efficiency, as shown in Fig. \ref{fig11c} and \ref{fig11d}. Finally, it can be found that none of the algorithms can present statistically comparable performance with AMFCM in both efficiency and quality. Therefore, AMFCM statistically achieves a good trade-off between the clustering of quality and efficiency.

\subsubsection{Efficiency of the New Affinity Filtering Scheme}\label{subsubsec5-3-3}
\begin{figure*}[htp!]
\centering
\includegraphics[width=0.96\textwidth]{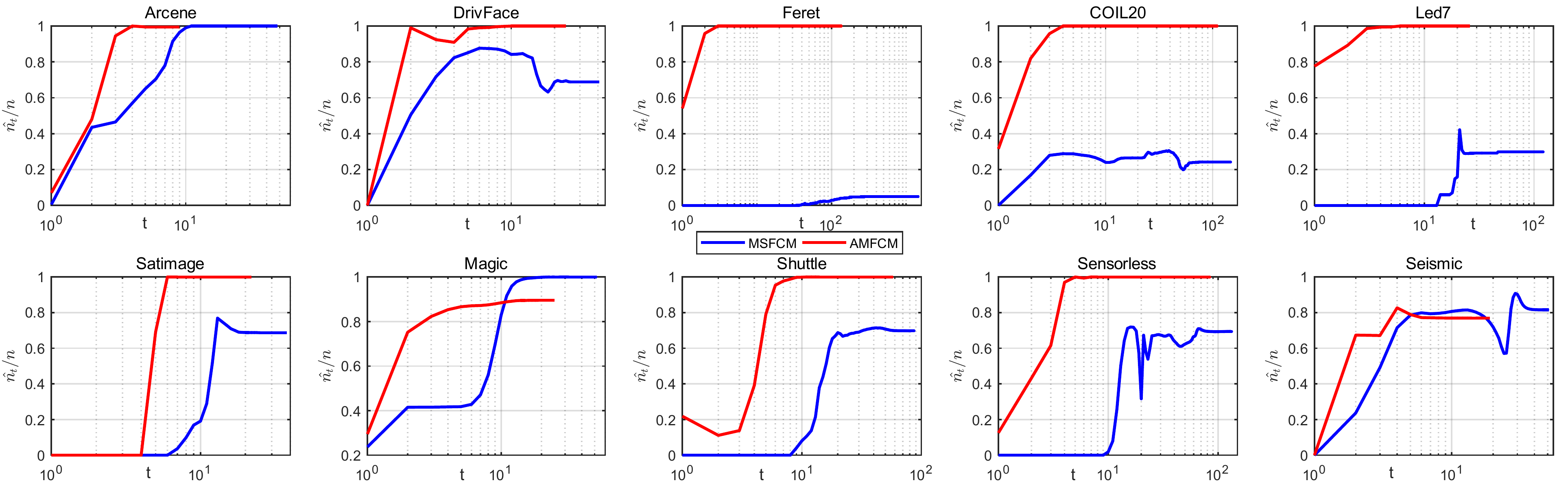}
\caption{Plots of $\frac{\hat{n}_{t}}{n}$ in relation to iteration on ten data sets. The log-scale of the x-axis clearly shows the differences of the sample filtering efficiency in [A] stage of the convergence process of the algorithms.}\label{fig8}
\end{figure*}

This set of experiments, carried out on the same data sets, is presented to test the filtering rate of the new affinity filtering \eqref{eq6}, which is the key factor that determines the acceleration of AMFCM.  The previous affinity filtering \eqref{eq_4} and the new affinity filtering \eqref{eq6} are chosen with the same settings for comparison. Let $\hat{n}_{t}=|\{1\leq j \leq n \mid |\mathcal{P}^{(t)}_{j}|\neq 0\}|$ denote the number of the samples that satisfy the affinity filtering in $t$ iteration. For ten data sets, the curves of $\frac{\hat{n}_{t}}{n}$ are plotted in relation to iteration $t$ in Fig. \ref{fig8}, where the blue and red lines represent MSFCM and AMFCM, respectively. And the log-scale of the x-axis clearly shows the differences of the sample filtering efficiency in stage [A] of the convergence process of the algorithms.


Firstly, the filtering efficiency of the new affinity filtering \eqref{eq6} is higher than that of the previous affinity filtering \eqref{eq_4} in stage [A]. Moreover, the new affinity filtering \eqref{eq6} has reached the highest filtering rate before 10 iterations, except for Magic and Seismic. Thus, the new affinity filtering \eqref{eq6} overcomes the inherent shortcomings of the previous affinity filtering \eqref{eq_4} that is easy to be invalid in stage [A], as shown in Fig. \ref{fig8}.

Secondly, for data sets Magic and Seismic, the efficiency of the previous affinity filtering \eqref{eq_4} is higher than that of AMFCM in stage [B] of the convergence process of the algorithms. However, the previous membership scaling \eqref{eq_new u} is not very effective in accelerating the algorithm in stage [B], which is because  $\beta_j^{(t)}$ is very close to 1 in stage [B]. The redundant contributions of the samples in the update of the centers have not been reduced in MSFCM, resulting in a decrease in its clustering efficiency.

AMFCM always maintains a high-efficiency level for all data sets, which can be observed from the number of iteration on the x-axis. According to the above analysis, it can be seen that the new affinity filtering \eqref{eq6} and membership scaling \eqref{eq_modified u} schemes are complementary to each other.  Therefore, AMFCM is very efficient in stages [A] and [B] of the convergence process of the algorithms.
\section{Conclusion}\label{sec7}
In this paper, FCM based on new affinity filtering and membership scaling (AMFCM) is proposed to accelerate the whole convergence stages of the traditional FCM clustering. In the proposed algorithm, a new affinity filtering is designed to obtain the complete non-affinity centers for each sample by a new set of triangle inequalities, which is more compatible with fuzzy clustering. Then, a new membership scaling is suggested to eliminate the contributions of the samples in the update of their non-affinity centers by setting the membership grades to 0 and promote the contributions of the samples in the update of the remaining centers by the fuzzy membership grades, which improves the performance and efficiency of the algorithm. Many experimental results have verified its effectiveness and efficiency on synthetic and real-world data sets. Therefore, AMFCM is a well-balanced FCM-type algorithm in clustering efficiency and quality. For future work, AMFCM could be explored to enhance the performance of the FCM-type clustering algorithms on high-dimensional data sets. Another interesting possibility is to generalize the concept to nonlinear fuzzy clustering according to information granules.

\bibliographystyle{IEEEtran}
\bibliography{Reference}

\clearpage

\begin{appendices}
\section{A geometric interpretation for the new affinity filtering scheme \eqref{eq6}}\label{sec7-1}
\begin{figure}[htp]
\centering
\includegraphics[width=0.33\textwidth]{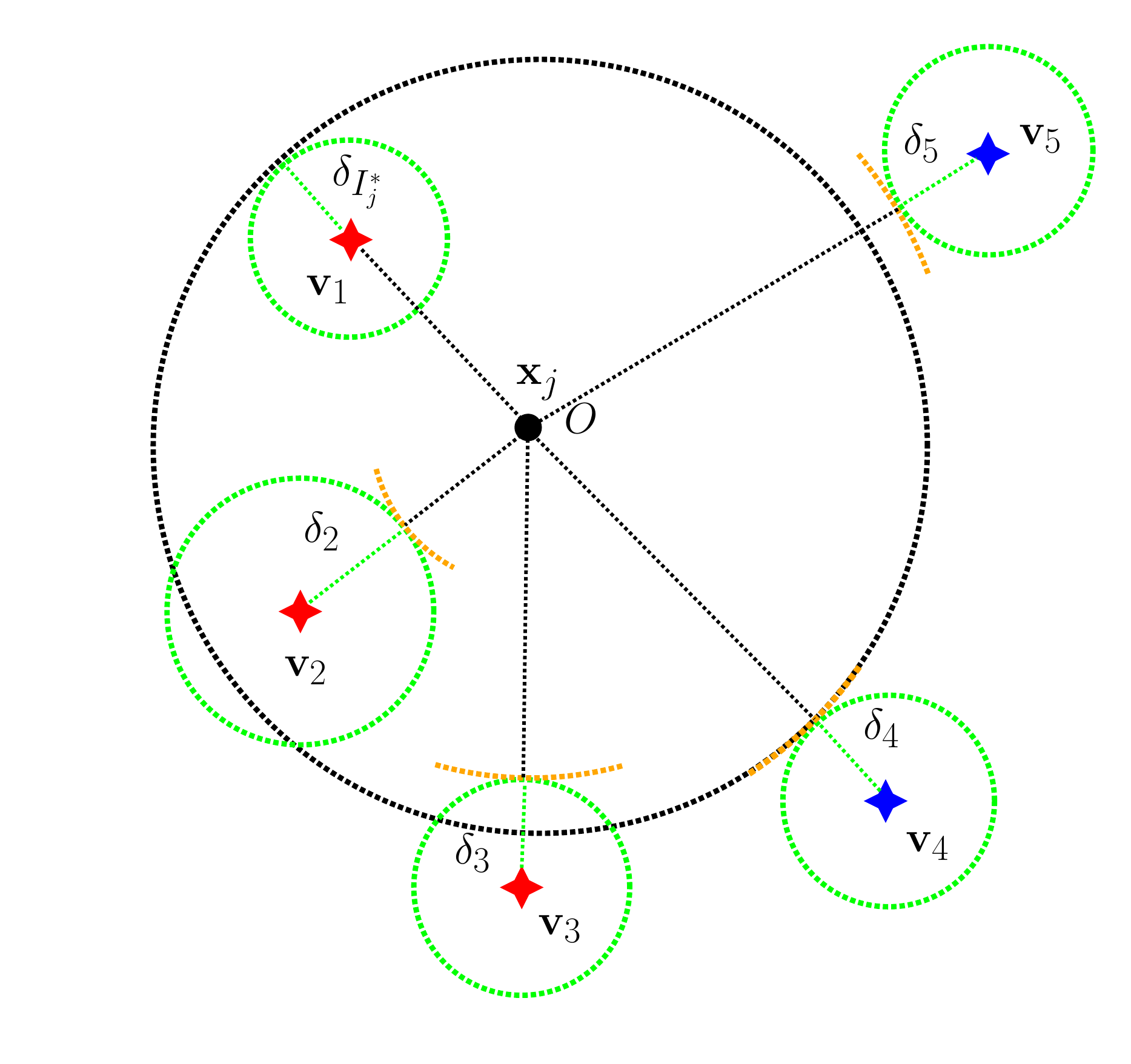}
\caption{The affinities between sample $\x_{j}$ and all centers are identified by the new affinity filtering scheme \eqref{eq6}. Let $\vv_{1}$ be the nearest center of $\x_{j}$.  The radius of the green circle and orange arc are the drift of $\vv_{i}$ and the lower bound of $\vv_{i}$ after one update, respectively.  The radius of the black circle is the upper bound of $\vv_{1}$ after one update. Here, the blue and red centers are the non-affinity and the remaining centers of $\x_{j}$, respectively.}\label{fig5}
\end{figure}

Here, a geometric interpretation is given for the new affinity filtering scheme \eqref{eq6}, as illustrated in Fig. \ref{fig5}. In Fig. \ref{fig5}, there is any sample $\x_{j}$ and the cluster centers are $\V=[\vv_{1}, \vv_{2}, \vv_{3}, \vv_{4}, \vv_{5}]$. Let $\vv_{1}$ be the nearest center of $\x_{j}$. The radius of green circle represent $\delta_{i}$, the radius of orange arc represent the lower bound of $\vv_{i}$, $d_{ij}-\delta_{i}$, and the radius of black circle is the upper bound of $\vv_{1}$, $d_{1,j}+\delta_{1}$. In Fig. \ref{fig5}, as long as the orange arc of $\vv_{i}$ is outside the black circle, it means that $\vv_{i}$ cannot be the nearest center of $\x_{j}$. Thus, $\vv_{i}$ is identified as the non-affinity center of $\x_{j}$ according to \textbf{Lemma} \ref{lem2}. For example, $\vv_{4}$ and $\vv_{5}$ are the non-affinity centers of $\x_{j}$, which are marked in blue.  The red centers are the remaining centers of $\x_{j}$. In this case, the set of the non-affinity centers of $\x_{j}$ obtained by \eqref{eq6} is complete.

\section{Proof of \textbf{Theorem} \ref{theorem3}}\label{sec7-4}
As analyzed in Subsection \ref{subsec3-2}, in the stage [B] of the convergence process, the assignment of most samples remains unchanged with the slight changes of the centers, where $\tilde{\mathbf{X}}^{(t)}$ is defined as the set of the samples whose the assignment does not change in $t$ and $t+1$ iteration. Therefore, we have: $\tilde{u}^{(t)}_{\cdot j}=\tilde{u}^{(t+1)}_{\cdot j}$ for $\x_{j} \in \tilde{\mathbf{X}}^{(t)}$ based on the Eq. \eqref{eq6} and \eqref{eq_modified u}.  However, in FCM, we have: $u_{\cdot j}^{(t)} \neq u_{\cdot j}^{(t+1)}$ for $\x_{j} \in \tilde{\mathbf{X}}^{(t)}$ based on the Eq. \eqref{eq_3}. According to the membership grade matrix $\U$ and $\tilde{\U}$ of FCM and AMFCM, we have:\\
\begin{equation}
\|\tilde{\U}^{(t+1)}-\tilde{\U}^{(t)}\| < \|\U^{(t+1)}-\U^{(t)}\|.
\end{equation}

Therefore, given the same termination parameter $\varepsilon$, the number of the iteration of AMFCM, $t_{\textbf{AMFCM}}$, is smaller than that of FCM, $t_{\textbf{FCM}}$.

\section{Proof of \textbf{Theorem} \ref{theorem1}}\label{sec7-2}
First, we list the objective function of FCM and AMFCM with the same initialization $\V^{(t)}$ from the $t$ iteration to the $t+1$ iteration:
\begin{equation}
\begin{aligned}
&J_{\textbf{FCM}}^{(t)}=\sum_{j=1}^{n} \sum_{i=1}^{c} ({u_{\text{FCM}}}^{(t)}_{ij})^{m} \|\x_{j}-{\vv}_{i}^{(t)}\|^2,\\
&J_{\textbf{FCM}}^{(t+1)}=\sum_{j=1}^{n} \sum_{i=1}^{c} ({u_{\text{FCM}}}^{(t)}_{ij})^{m} \|\x_{j}-{\vv_{\text{FCM}}}_{i}^{(t+1)}\|^2,\\
&J_{\textbf{AMFCM}}^{(t)}=\sum_{j=1}^{n} \sum_{i=1}^{c} (\tilde{u}^{(t)}_{ij})^m \|\x_{j}-{\vv}_{i}^{(t)}\|^2,\\
&J_{\textbf{AMFCM}}^{(t+1)}=\sum_{j=1}^{n} \sum_{i=1}^{c} (\tilde{u}^{(t)}_{ij})^m \|\x_{j}-{\vv_{\text{AMFCM}}}_{i}^{(t+1)}\|^2.\\
\end{aligned}
\label{eqchange}
\end{equation}

We assume that there must exist an iteration threshold  $T$, when $t<T$, $\mathcal{P}_{j}^{(t)}=\emptyset$ for $j=1,2,...,n$. Obviously, when $t<T$, $J_{\textbf{FCM}}=J_{\textbf{AMFCM}}$. When $t \geq T$, $\exists j \in \{1,2,...,n\}$, satisfying $\mathcal{P}_{j}^{(t)} \neq \emptyset$. Then, we derive lower and upper bounds of $J_{\textbf{AMFCM}}$.

For the lower bound of $J_{\textbf{AMFCM}}^{(t)}$, based on the Lagrangian multiplier method, we have:
$J_{\textbf{FCM}}^{(t)}<J_{\textbf{AMFCM}}^{(t)}$. For the upper bound of $J_{\textbf{AMFCM}}^{(t)}$, we have:

\begin{equation}
\begin{aligned}
J_{\textbf{AMFCM}}^{(t)}&=\sum_{j=1}^{n} \sum_{i=1}^{c} (\tilde{u}^{(t)}_{ij})^m \|\x_{j}-{\vv}_{i}^{(t)}\|^2,\\
&=\sum_{j=1}^{n} [\alpha^{(t)}_{j}]^{m} \cdot [\sum_{i \notin \mathcal{P}_{j}^{(t)} } ({u_{\text{FCM}}}^{(t)}_{ij})^m \|\x_{j}-{\vv}_{i}^{(t)}\|^2],\\
&< [\max_{1\le j\le n} \{\alpha^{(t)}_{j}\}]^{m} \cdot [\sum_{j=1}^{n} \sum_{i=1}^{c} ({u_{\text{FCM}}}^{(t)}_{ij})^{m} \|\x_{j}-{\vv}_{i}^{(t)}\|^2],\\
&=[\max_{1\le j\le n} \{\alpha^{(t)}_{j}\}]^{m} \cdot J_{\textbf{FCM}}^{(t)},\\
\end{aligned}
\label{eqchange1}
\end{equation}
where  $\tilde{\U}^{(t)}$ can be rewritten as $\alpha^{(t)}_{j} \cdot {u_{\text{FCM}}}^{(t)}_{ij}$ for $i\notin \mathcal{P}_{j}^{(t)}$, and $0$ for $i \in \mathcal{P}_{j}^{(t)}$. Based on the Eq. \eqref{eq6} and \eqref{eq_modified u}, $1<\alpha^{(t)}_{j}={\frac{1}{1-\sum_{i \in \mathcal{P}_{j}^{(t)}} {u_{\text{FCM}}}^{(t)}_{ij}}}<2.$  Therefore, we have:

\begin{equation}
J_{\textbf{FCM}}^{(t)}<J_{\textbf{AMFCM}}^{(t)}<[\max_{1\le j\le n} \{\alpha^{(t)}_{j}\}]^{m} \cdot J_{\textbf{FCM}}^{(t)}.
\label{eqchange2}
\end{equation}

For the upper bound of $J_{\textbf{AMFCM}}^{(t+1)}$, based on the Lagrangian multiplier method, we have:
\begin{equation}
\begin{aligned}
J_{\textbf{AMFCM}}^{(t+1)}=\sum_{j=1}^{n} \sum_{i=1}^{c} (\tilde{u}^{(t)}_{ij})^m \|\x_{j}-{\vv_{\text{AMFCM}}}_{i}^{(t+1)}\|^2,\\
< \sum_{j=1}^{n} \sum_{i=1}^{c} (\tilde{u}^{(t)}_{ij})^m \|\x_{j}-{\vv_{\text{FCM}}}_{i}^{(t+1)}\|^2.\\
\end{aligned}
\label{eqchange3}
\end{equation}

Based on a similar derivation from the Eq. \eqref{eqchange1}, we have:
\begin{equation}
J_{\textbf{AMFCM}}^{(t+1)}<[\max_{1\le j\le n} \{\alpha^{(t)}_{j}\}]^{m} \cdot J_{\textbf{FCM}}^{(t+1)}.
\label{eqchange4}
\end{equation} 

Finally, based on the Eq. \eqref{eqchange2} and  \eqref{eqchange4}, we have:
\begin{equation}
\frac{J_{\textbf{AMFCM}}^{(t)}}{J_{\textbf{AMFCM}}^{(t+1)}} > \frac{J_{\textbf{FCM}}^{(t)}}{[\max_{1\le j\le n} \{\alpha^{(t)}_{j}\}]^{m} \cdot J_{\textbf{FCM}}^{(t+1)}}.
\label{eqchange5}
\end{equation} 

Here,  we set the  decline ratio of the objective function of FCM as $\lambda^{(t)}=\frac{J_{\textbf{FCM}}^{(t)}}{J_{\textbf{FCM}}^{(t+1)}}$. As analyzed in \cite{du1999centroidal}, the convergence rate of the alternating optimization algorithm (AO) drops with the iteration.  Therefore, $\lambda^{(t)}$ is monotonically decreasing with the iteration, i.e., $\lambda^{(t)} \geq 1$ and $\lim\limits_{t\to +\infty} \lambda^{(t)}=1$. Obviously, in the mid stage, when $\lambda^{(t)}>[\max_{1\le j\le n} \{\alpha^{(t)}_{j}\}]^{m}$, we have:
\begin{equation}
\frac{J_{\textbf{AMFCM}}^{(t)}}{J_{\textbf{AMFCM}}^{(t+1)}} >1.
\label{eqchange7}
\end{equation} 

In conclusion, AMFCM does not converge precociously in the mid stage of the clustering process.

\section{Proof of \textbf{Theorem} \ref{theorem2}}\label{sec7-3}
First,  let ${\hat{\mathbf{X}}}^{(t)}$ be the set of samples $\x_{j}$ with $|\mathcal{P}_{j}^{(t)}|=1$ in $t$ iteration, where $j=1,2,...,|\hat{\mathbf{X}}^{(t)}|$.
Given the same initialization $\V^{(t)}$ on $\hat{\mathbf{X}}^{(t)}$, obviously, ${J_{\textbf{Hard}}^{\textbf{AMFCM}}}^{(t+1)}={J_{\textbf{Hard}}^{\textbf{FCM}}}^{(t+1)}$. Next, after one iteration of centers $\V^{(t)}$ through FCM and AMFCM respectively, based on the Lagrangian multiplier method,  we have:\\
\begin{equation}
\begin{aligned}
J_{\textbf{AMFCM}}^{(t+1)}&=\sum_{j=1}^{|\hat{\mathbf{X}}^{(t)}|}\sum_{i=1}^{c} (\tilde{u}^{(t)}_{ij})^m \|x_{j}-{\vv_{\text{AMFCM}}}_{i}^{(t+1)}\|^2 \\
&< \sum_{j=1}^{|\hat{\mathbf{X}}^{(t)}|}\sum_{i=1}^{c} (\tilde{u}^{(t)}_{ij})^m \|x_{j}-{\vv_{\text{FCM}}}_{i}^{(t+1)}\|^2.
\end{aligned}
\end{equation}

Because $|\mathcal{P}^{(t)}_{j}|=1$ for any $\x_{j} \in \hat{\mathbf{X}}^{(t)}$, we have that $\tilde{u}^{(t)}_{I_{j}^{*}j}=1$, and $\tilde{u}^{(t)}_{ij}=0$ for $i \neq I_{j}^{*}$. Therefore, we have: 

\begin{equation}
\begin{aligned}
{J_{\textbf{Hard}}^{\textbf{AMFCM}}}^{(t+1)}=J_{\textbf{AMFCM}}^{(t+1)} <{J_{\textbf{Hard}}^{\textbf{FCM}}}^{(t+1)}.
\end{aligned}
\end{equation}

In conclusion, for the samples $\x_{j}$ with $|\mathcal{P}_{j}^{(t)}|=1$, the corresponding hard objective value of AMFCM is smaller than that of FCM in $t$ iteration.

\end{appendices}

\end{document}